\documentclass[12pt]{colt2020} 

\title[Loss Functions, Axioms, and Peer Review]{Loss Functions, Axioms, and Peer Review}
\usepackage{times}

\hypersetup{
    colorlinks = true,
 	linkcolor = teal,
 	citecolor = teal
}

\usepackage{tikz}
\usepackage{caption}

\newtheorem{claim}{Claim}

\renewcommand{\vec}{\mathbf}
\newcommand{\crit}{{\vec x}}
\renewcommand{\over}{y}
\newcommand{\aggr}{\widehat{f}}

\newcommand{\domEach}{\mathcal{X}}
\newcommand{\domY}{\mathbb{Y}}
\newcommand{\domX}{\mathbb{X}}
\newcommand{\fclass}{\mathcal{F}}
\newcommand{\truef}{g^\star}
\newcommand{\Lpq}[2]{$L(#1,#2)$}
\newcommand{\optOver}{y_{\optfn}}
\newcommand{\allP}{\mathcal{P}}
\newcommand{\allR}{\mathcal{R}}
\newcommand{\normDom}{[1,\infty]}
\newcommand{\infopt}{\widehat{v}}

\allowdisplaybreaks

\coltauthor{%
 \Name{Ritesh Noothigattu} \Email{riteshn@cmu.edu}\\
 \addr School of Computer Science, Carnegie Mellon University
 \AND
 \Name{Nihar B. Shah} \Email{nihars@cs.cmu.edu}\\
 \addr School of Computer Science, Carnegie Mellon University
 \AND
 \Name{Ariel D. Procaccia} \Email{arielpro@seas.harvard.edu}\\
 \addr School of Engineering and Applied Sciences, Harvard University
}

\begin{document}

\maketitle

\begin{abstract}
It is common to see a handful of reviewers reject a highly novel paper, because they view, say, extensive experiments as far more important than novelty, whereas the community as a whole would have embraced the paper. More generally, the disparate mapping of criteria scores to final recommendations by different reviewers is a major source of inconsistency in peer review. In this paper we present a framework inspired by empirical risk minimization (ERM) for learning the community's aggregate mapping. The key challenge that arises is the specification of a loss function for ERM. We consider the class of $L(p,q)$ loss functions, which is a matrix-extension of the standard class of $L_p$ losses on vectors; here the choice of the loss function amounts to choosing the hyperparameters $p, q \in [1,\infty]$. To deal with the absence of ground truth in our problem, we instead draw on computational social choice to identify desirable values of the hyperparameters $p$ and $q$. Specifically, we characterize $p=q=1$ as the only choice of these hyperparameters that satisfies three natural axiomatic properties. Finally, we implement and apply our approach to reviews from IJCAI 2017.  
\end{abstract}

\section{Introduction}

The essence of science is the search for objective truth, yet scientific work is typically evaluated through peer review\footnote{Even papers about peer review are subject to peer review, the irony of which has not escaped us.}\,---\,a notoriously subjective process~\citep{church2005reviewing,lamont2009professors,bakanic1987manuscript,hojat2003impartial,mahoney1977publication,kerr1977manuscript}. One prominent source of subjectivity is the disparity across reviewers in terms of their emphasis on the various criteria used for the overall evaluation of a paper. \citet{Lee15} refers to this disparity as \emph{commensuration bias}, and describes it as follows:
\begin{quote}
\emph{``In peer review, reviewers, editors, and grant program officers must make interpretive decisions about how to weight the relative importance of qualitatively different peer review criteria\,---\,such as novelty, significance, and methodological soundness\,---\,in their assessments of a submission's final/overall value. Not all peer review criteria get equal weight; further, weightings can vary across reviewers and contexts even when reviewers are given identical instructions.''}
\end{quote}
\citet{Lee15} further argues that commensuration bias ``illuminates how intellectual priorities in individual peer review judgments can collectively subvert the attainment of community-wide goals'' and that it ``permits and facilitates problematic patterns of publication and funding in science.'' There have been, however, very few attempts to address this problem.

A fascinating exception, which serves as a case in point, is the 27th AAAI Conference on Artificial Intelligence (AAAI 2013). Reviewers were asked to score papers, on a scale of 1--6, according to the following criteria: technical quality, experimental analysis, formal analysis, clarity/presentation, novelty of the question, novelty of the solution, breadth of interest, and potential impact. The admirable goal of the program chairs was to select ``exciting but imperfect papers'' over ``safe but solid'' papers, and, to this end, they provided detailed instructions on how to map the foregoing criteria to an overall recommendation. For example, the preimage of `strong accept' is ``a 5 or 6 in some category, no 1 in any category,'' that is, reviewers were instructed to strongly accept a paper that has a 5 or 6 in, say, clarity, but is below average according to each and every other criterion (i.e., a clearly boring paper).
It turns out that the handcrafted mapping did not work well, and many of the reviewers chose to not follow these instructions.
Indeed, handcrafting such a mapping requires specifying an 8-dimensional function, which is quite a non-trivial task. Consequently, in this paper we do away with a manual handcrafting approach to this problem.

Instead, we propose a data-driven approach based on machine learning, designed to learn a mapping from criteria scores to recommendations capturing the opinion of the entire (reviewer) community. From a machine learning perspective, the examples are reviews, each consisting of criteria scores (the input point) and an overall recommendation (the label). We make the innocuous assumption that each reviewer has a \emph{monotonic} mapping in mind, in the sense that a paper whose scores are at least as high as those of another paper on every criterion would receive an overall recommendation that is at least as high; the reviews submitted by a particular reviewer can be seen as observations of that mapping. Given this data, our goal is to learn a \emph{single monotonic mapping} that minimizes a loss function (which we discuss momentarily).
We can then apply this mapping to the criteria scores associated with each review to obtain new overall recommendations, which replace the original ones. 

Our approach to learn this mapping is inspired by empirical risk minimization (ERM). In more detail, for some loss function, our approach is to find a mapping that, among all monotonic mappings from criteria scores to the overall scores, minimizes the loss between its outputs and the overall scores given by reviewers across all reviews. However, the choice of loss function may significantly affect the final outcome, so that choice is a key issue.

Specifically, we focus on the family of $L(p,q)$ loss functions, with hyperparameters $p,q\in \normDom$, which is a matrix-extension of the more popular family of $L_p$ losses on vectors. Our question, then, is: 
\begin{quote}
What values of the hyperparameters $p\in \normDom$ and $q\in \normDom$ in the specification of the \Lpq{p}{q} loss function should be used? 
\end{quote}

A challenge we must address is the absence of any ground truth in peer review. To this end, take the perspective of  \emph{computational social choice}~\citep{BCEL+16}, since our framework aggregates individual opinions over mappings into a consensus mapping. From this viewpoint, it is natural to select the loss function so that the resulting aggregation method satisfies socially desirable properties, such as  \emph{consensus} (if all reviewers agree then the aggregate mapping should coincide with their recommendations),  \emph{efficiency} (if one paper dominates another then its overall recommendation should be at least as high), and \emph{strategyproofness} (reviewers cannot pull the aggregate  mapping closer to their own recommendations by misreporting them).

With this background, the main contributions of this paper are as follows. We first provide a principled framework for addressing the issue of subjectivity regarding the various criteria in peer review. 

Our main theoretical result is a characterization theorem that gives a decisive answer to the question of choosing the loss function for ERM: the three aforementioned properties are satisfied \emph{if and only if} the hyperparameters are set as $p=q=1$. This result singles out an instantiation of our approach that we view as particularly attractive and well grounded. 

We also provide empirical results, which analyze properties of our approach when applied to a dataset of $9197$ reviews from IJCAI 2017. One vignette is that the papers selected by \Lpq{1}{1} aggregation have a $79.2\%$ overlap with actual list of accepted papers, suggesting that our approach makes a significant difference compared to the status quo (arguably for the better).

Finally, we note that the approach taken in this paper may find other applications. Indeed, the problem of selecting a loss function is ubiquitous in machine learning~\citep{RVCP+04,MV08,MBM18}, and the axiomatic approach provides a novel way of addressing it. Going beyond loss functions, machine learning researchers frequently face the difficulty of picking an appropriate hypothesis class or values for certain hyperparameters.\footnote{Popular techniques such as cross-validation for choosing hyperparameters also in turn depend on specification of a loss function.} Thus, in problem settings where such choices must be made\,---\,particularly in emerging applications of machine learning (such as peer review)\,---\,the use of natural axioms can help guide these choices.

\section{Our Framework}
\label{sec:frame}

Suppose there are $n$ reviewers $\allR=\{1,2, \dots, n\}$, and a set $\allP$ of $m$ papers, denoted using letters such as $a,b,c$. Each reviewer $i$ reviews a subset of papers, denoted by $P(i) \subseteq \allP$. Conversely, let $R(a)$ denote the set of all reviewers who review paper $a$. Each reviewer assigns scores to each of their papers on $d$ different criteria, such as novelty, experimental analysis, and technical quality, and also gives an overall recommendation. We denote the \emph{criteria scores} given by reviewer $i$ to paper $a$ by $\crit_{ia}$, and the corresponding \emph{overall recommendation} by $\over_{ia}$. Let $\domEach_1, \domEach_2, \dots, \domEach_d$ denote the domains of the $d$ criteria scores, and let $\domX = \domEach_1 \times \domEach_2 \times \dots \times \domEach_d$. Also, let $\domY$ denote the domain of the overall recommendations. For concreteness, we assume that each $\domEach_k$ as well as $\domY$ is the real line. However, our results hold more generally, even if these domains are non-singleton intervals in $\mathbb{R}$, for instance.

We further assume that each reviewer has a monotonic function in mind that they use to compute the overall recommendation for a paper from its criteria scores. By a monotonic function, we mean that given any two score vectors $\crit$ and $\crit'$, if $\crit$ is greater than or equal to $\crit'$ on all coordinates, then the function's value on $\crit$ must be at least as high as its value on $\crit'$. Formally, for each reviewer $i$, there exists $\truef_i \in \fclass$ such that $\over_{ia} = \truef_i(\crit_{ia})$ for all $a \in P(i)$, where
{\small
\[\fclass = \{f:\domX \to \domY \ | \ \forall \crit, \crit' \in \domX, \crit \geq \crit' \Rightarrow f(\crit) \geq f(\crit')\}\]
}
is the set of all monotonic functions.

\subsection{Loss Functions}

Recall that our goal is to use all criteria scores, and their corresponding overall recommendations, to learn an aggregate function $\aggr$ that captures the opinions of all reviewers on how criteria scores should be mapped to recommendations. Inspired by empirical risk minimization, we do this by computing the function in $\fclass$ that minimizes the \Lpq{p}{q} loss on the data. In more detail, given hyperparameters $p,q \in \normDom$, we compute
\begin{align}   \label{eqn:aggregation}
\aggr \in \underset{f \in \fclass}{\text{argmin}}\ \left\{ \sum_{i\in \allR} \left[ \sum_{a \in P(i)} \left|y_{ia} - f(\vec{x}_{ia})\right|^p \right]^{\frac{q}{p}} \right\}^{\frac{1}{q}}.
\end{align}

In words, for a function $f$, the \Lpq{p}{q} loss is the $L_q$ norm taken over the loss associated with individual reviewers, where the latter loss is defined as the $L_p$ norm computed on the error of $f$ with respect to the reviewer's overall recommendations. The \Lpq{p}{q} loss is a matrix-extension of the more popular $L_p$ losses on vectors, and relates to the \Lpq{p}{q} norm of a matrix which has had many applications in machine learning~\citep{ding2006r,kong2011robust,nie2010efficient}. We refer to aggregation by minimizing \Lpq{p}{q} loss as defined in Equation~\eqref{eqn:aggregation} as \emph{\Lpq{p}{q} aggregation}.

Equation~\eqref{eqn:aggregation} does not specify how to break ties between multiple minimizers. For concreteness, we select the minimizer $\aggr$ with minimum empirical $L_2$ norm. Formally, letting
\[\widehat{F} = \underset{f \in \fclass}{\text{argmin}}\ \left\{ \sum_{i\in \allR} \left[ \sum_{a \in P(i)} \left|y_{ia} - f(\vec{x}_{ia})\right|^p \right]^{\frac{q}{p}} \right\}^{\frac{1}{q}}\]
be the set of all \Lpq{p}{q} loss minimizers, we break ties by choosing
\begin{align}   \label{eqn:tie-break}
\aggr \in \underset{f \in \widehat{F}}{\text{argmin}}\ \sqrt{\sum_{i \in \allR} \sum_{a \in P(i)} f(\vec{x}_{ia})^2}.
\end{align}
Observe that since the \Lpq{p}{q} loss and constraint set are convex, $\widehat{F}$ is also a convex set. Hence, $\aggr$ as defined by Equation~\eqref{eqn:tie-break} is unique. We emphasize that although we use minimum $L_2$ norm for tie-breaking, all of our results hold under any reasonable tie-breaking method, such as the minimum $L_k$ norm for any $k \in (1, \infty)$.

Once the function $\aggr$ has been computed, it can be applied to every review (for all reviewers $i$ and papers $a$) to obtain a new overall recommendation $\aggr(\vec{x}_{ia})$. There is a separate\,---\,almost orthogonal\,---\,question of how to aggregate the overall recommendations of several reviewers on a paper into a single recommendation (typically this is done by taking the average). In our theoretical results we are agnostic to how this additional aggregation step is performed, but we return to it in our experiments in Section~\ref{sec:empirical}.

We remark that an alternative approach would be to learn a monotonic function $\widehat{g}_i:\domX\rightarrow \domY$ for each reviewer (which best captures their recommendations), and then aggregate these functions into a single function $\aggr$. We chose not to pursue this approach, because in practice there are very few examples per reviewer, so it is implausible that we would be able to accurately learn the reviewers' individual functions.

\subsection{Axiomatic Properties}
\label{subsec:axioms}

In social choice theory, the most common approach\,---\, primarily attributed to \citet{Arr51}\,---\,for comparing different aggregation methods is to determine which desirable axioms they satisfy. We take the same approach in order to determine the values of the hyperparameters $p$ and $q$ for the \Lpq{p}{q} aggregation in Equation~\eqref{eqn:aggregation}.

We stress that axioms are defined for aggregation methods and not aggregate functions. Informally, an aggregation method is a function that takes as input all the reviews $\{(\crit_{ia}, y_{ia})\}_{i \in \allR, a \in P(i)}$, and outputs an aggregate function $\aggr: \domX \to \domY$. We do not define an aggregation method formally to avoid introducing cumbersome notation that will largely be useless later. It is clear that for any choice of hyperparameters $p,q \in \normDom$, \Lpq{p}{q} aggregation (with tie-breaking as defined by Equation~\ref{eqn:tie-break}) is an aggregation method.

Social choice theory essentially relies on counterfactual reasoning to identify scenarios where it is clear how an aggregation method should behave.  To give one example, the \emph{Pareto efficiency} property of voting rules states that if all voters prefer alternative $a$ to alternative $b$, then $b$ should \emph{not} be elected; this situation is extremely unlikely to occur, yet Pareto efficiency is obviously a property that any reasonable voting must satisfy. With this principle in mind, we identify a setting in our problem where the requirements are very clear, and then define our axioms in that setting.

For all of our axioms, we restrict attention to scenarios where every reviewer reviews every paper, that is, $P(i) = \allP$ for every $i$. Moreover, we assume that the papers have `objective' criteria scores, that is, the criteria scores given to a paper are the same across all reviewers, so the only source of disagreement is how the criteria scores should be mapped to an overall recommendation. We can then denote the criteria scores of a paper $a$ simply as $\crit_a$, as opposed to $\crit_{ia}$, since they  do not depend on $i$. We stress that our framework does \emph{not} require these assumptions to hold\,---\,they are only used in our axiomatic characterization, namely Theorem~\ref{thm:master} in the next section. 

An axiom is satisfied by an aggregation method if its statement holds for every possible number of reviewers $n$ and number of papers $m$, and for all possible criteria scores and overall recommendations. We start with the simplest axiom, consensus, which informally states that if there is a paper such that all reviewers give it the same overall recommendation, then $\aggr$ must agree with the reviewers; this axiom is closely related to the \emph{unanimity} axiom in social choice.  

\begin{axiom}[Consensus]
For any paper $a \in \allP$, if all reviewers report identical overall recommendations $y_{1a}=y_{2a}=\cdots = y_{m a} = r$  for some $r\in \domY$, then $\aggr(\crit_a)=r$.
\end{axiom}

Before presenting the next axiom, we require another definition: we say that paper $a \in \allP$ \emph{dominates} paper $b \in \allP$ if there exists a bijection $\sigma: \allR \to \allR$ such that for all $i \in \allR$, $y_{ia} \geq y_{\sigma(i)b}$. Equivalently (and less formally), paper $a$ dominates paper $b$ if the \emph{sorted} overall recommendations given to $a$ pointwise-dominate the \emph{sorted} overall recommendations given to $b$. Intuitively, in this situation, $a$ should receive a (weakly) higher overall recommendation than $b$, which is exactly what the axiom requires; it is similar to the classic Pareto efficiency axiom mentioned above. 

\begin{axiom}[Efficiency]
For any pair of papers $a,b \in \allP$, if $a$ dominates $b$,  then $\aggr(\crit_a)\geq \aggr(\crit_b)$. 
\end{axiom}
Our positive result, which will be presented shortly, satisfies this notion of efficiency. On the other hand, we also use this axiom to prove a negative result; an important note is that the negative result requires a condition that is significantly weaker than the aforementioned  definition of efficiency. We revisit this point at the end of Section~\ref{sec:efficiency_violated}.

Our final axiom is \emph{strategyproofness}, a game-theoretic property that plays a major role in social choice theory~\citep{Moul83}. Intuitively, strategyproofness means that reviewers have no incentive to misreport their overall recommendations: They cannot bring the aggregate recommendations\,---\,the community's consensus about the relative importance of various criteria\,---\,closer to their own through strategic manipulation.

\begin{axiom}[Strategyproofness]
For each reviewer $i\in \allR$, and all possible manipulated recommendations $\vec{\over}_i' \in \domY^m$, if $\vec{\over}_i = (\over_{i1}, \over_{i2}, \dots, \over_{im})$ is replaced with $\vec{\over}_i'$, then
\begin{equation} 
\label{EqnStrategyproofness}
\|(\aggr(\crit_1),\ldots,\aggr(\crit_m)) - \vec{\over}_i\|_2\leq \|(\widehat{g}(\crit_1),\ldots,\widehat{g}(\crit_m)) - \vec{\over}_i\|_2,
\end{equation}
where $\aggr$ and $\widehat{g}$ are the aggregate functions obtained from the original and manipulated reviews, respectively.
\end{axiom}
The implicit `utilities' in this axiom~\eqref{EqnStrategyproofness} are defined in terms of the $L_2$ norm. This choice is made only for concreteness, and all our results hold for any norm $L_\ell$, $\ell \in \normDom$, in the definition~\eqref{EqnStrategyproofness}.

\section{Main Result}  \label{sec:results}

In Section~\ref{sec:frame}, we introduced \Lpq{p}{q} aggregation as a family of rules for aggregating individual opinions towards a consensus mapping from criteria scores to recommendations. But that definition, in and of itself, leaves open the question of how to choose the values of $p$ and $q$ in a way that leads to the most socially desirable outcomes. The axioms of Section~\ref{subsec:axioms} allow us to give a satisfying answer to this question. Specifically, our main theoretical result is a characterization of $L(p,q)$ aggregation in terms of the three axioms.  

\begin{theorem} \label{thm:master}
\Lpq{p}{q} aggregation, where $p,q \in \normDom$, satisfies consensus, efficiency, and strategyproofness if and only if $p = q = 1$.
\end{theorem}

We remark that for $p=q$, Equation~\eqref{eqn:aggregation} does not distinguish between different reviewers, that is, the aggregation method pools all reviews together. We find this interesting, because the \Lpq{p}{q} aggregation framework does have enough power to make that distinction, but the axioms guide us towards a specific solution, \Lpq{1}{1}, which does not.

Turning to the proof of the theorem, we start from the easier `if' direction.

\subsection{$p=q=1$ Satisfies All Three Axioms}

\begin{lemma} \label{lem:positive}
\Lpq{p}{q} aggregation with $p = q = 1$ satisfies consensus, efficiency and strategyproofness.
\end{lemma}

\begin{proof}
The key idea of the proof lies in the form taken by the minimizer of \Lpq{1}{1} loss. When each reviewer reviews every paper and the papers have objective criteria scores, \Lpq{1}{1} aggregation reduces to computing
\begin{align}   \label{eqn:l11-opt}
\aggr \in \underset{f \in \fclass}{\text{argmin}}\ \left\{ \sum_{i\in \allR} \sum_{a \in \allP} \left|y_{ia} - f(\vec{x}_{a})\right| \right\},
\end{align}
where ties are broken by picking the minimizer with minimum $L_2$ norm. We claim that the aggregate function is given by
\[\aggr(\crit_a) = \text{left-med}(\{y_{ia}\}_{i \in \allR}) \quad \forall a \in \allP,\]
where $\text{left-med}(\cdot)$ of a set of points is their left median. We prove this claim by showing four parts: 
\begin{itemize}
    \item[(i)] $\aggr$ is a valid function,
    \item[(ii)] $\aggr$ is an unconstrained minimizer of the objective in~\eqref{eqn:l11-opt},
    \item[(iii)] $\aggr$ satisfies the constraints of~\eqref{eqn:l11-opt}, i.e., $\aggr \in \fclass$, and
    \item[(iv)] $\aggr$ has the minimum $L_2$ norm among all minimizers of~\eqref{eqn:l11-opt}.
\end{itemize}

We start by proving part (i). This part can only be violated if there are two papers $a$ and $b$ such that $\crit_a = \crit_b$, but $\text{left-med}(\{y_{ia}\}_{i \in \allR}) \neq \text{left-med}(\{y_{ib}\}_{i \in \allR})$, leading to $\aggr$ having two function values for the same $\crit$-value. However, we assumed that each reviewer $i$ has a function $\truef_i$ used to score the papers. So, for the two papers $a$ and $b$, we would have $y_{ia} = \truef_i(\crit_a) = \truef_i(\crit_b) = y_{ib}$ for every $i$, giving us $\text{left-med}(\{y_{ia}\}_{i \in \allR}) = \text{left-med}(\{y_{ib}\}_{i \in \allR})$. Therefore, $\aggr$ is a valid function.

For part (ii), consider the optimization problem~\eqref{eqn:l11-opt} without any constraints. Denote the objective function as $G(f)$. Rearranging terms, we obtain
\begin{align}   \label{eqn:l11-obj}
G(f) = \sum_{a \in \allP} \sum_{i\in \allR} \left|y_{ia} - f(\vec{x}_{a})\right|.
\end{align}
Consider the inner summation $\sum_{i\in \allR} \left|y_{ia} - f(\vec{x}_{a})\right|$; it is obvious that this quantity is minimized when $f(\crit_a)$ is any median of the $\{y_{ia}\}_{i \in 
\allR}$ values. Hence, we have
\begin{equation}   
\label{eqn:median-opt}
\begin{split}
G(f) &= \sum_{a \in \allP} \sum_{i\in \allR} \left|y_{ia} - f(\vec{x}_{a})\right|\\
& \geq \sum_{a \in \allP} \sum_{i\in \allR} \left|y_{ia} - \text{left-med}(\{y_{ia}\}_{i \in \allR})\right|\\
& = G(\aggr),
\end{split}
\end{equation}
where $f$ is an arbitrary function. Therefore, $\aggr$ minimizes the objective function even in the absence of any constraints, proving part (ii).

Turning to part (iii), we show that $\aggr$ satisfies the monotonicity constraints, i.e., $\aggr \in \fclass$. Suppose $a, b \in \allP$ are such that $\crit_a \geq \crit_b$. Using the fact that each reviewer $i$ scores papers based on the function $\truef_i$, we have $y_{ia} = \truef_i(\crit_a)$ and $y_{ib} = \truef_i(\crit_b)$. And since $\truef_i \in \fclass$ obeys monotonicity constraints, we obtain $y_{ia} \geq y_{ib}$ for every $i$. This trivially implies that $\text{left-med}(\{y_{ia}\}_{i \in \allR}) \geq \text{left-med}(\{y_{ib}\}_{i \in \allR})$, i.e., $\aggr(\crit_a) \geq \aggr(\crit_b)$, completing part (iii).

Finally, we prove part (iv). Observe that Equation~\eqref{eqn:median-opt} is a strict inequality if there is a paper $a$ for which $f(\crit_a)$ is not a median of the $\{y_{ia}\}_{i \in \allR}$ values. In other words, the only functions $f$ that have the same objective function value as $\aggr$ are of the form
\begin{align}   \label{eqn:medians}
f(\crit_a) \in \text{med}(\{y_{ia}\}_{i \in \allR}) \quad \forall a \in \allP,
\end{align}
where $\text{med}(\cdot)$ of a collection of points is the set of all points between (and including) the left and right medians. Hence, all other minimizers of~\eqref{eqn:l11-opt} must satisfy Equation~\eqref{eqn:medians}. Observe that $\aggr$ is pointwise smaller than any of these functions, since it computes the left median at each of the $\crit$-values. Therefore, $\aggr$ has the minimum $L_2$ norm among all possible minimizers of~\eqref{eqn:l11-opt}, completing the proof of part (iv). 

Combining all four parts proves that $\aggr$ is indeed the aggregate function chosen by \Lpq{1}{1} aggregation. We use this to prove that \Lpq{1}{1} aggregation satisfies consensus, efficiency and strategyproofness.

\emph{Consensus.} Let $a \in \allP$ be a paper such that $y_{1a} = y_{2a} = \cdots = y_{ma} = r$ for some $r$. Then, $\text{left-med}(\{y_{ia}\}_{i \in \allR}) = r$. Hence, $\aggr(\crit_a) = r$, satisfying consensus.

\emph{Efficiency.} Let $a, b \in \allP$ be such that $a$ dominates $b$. In other words, the sorted overall recommendations given to $a$ pointwise-dominate the sorted overall recommendations given to $b$. So, by definition, $\text{left-med}(\{y_{ia}\}_{i \in \allR})$ is at least as large as $\text{left-med}(\{y_{ib}\}_{i \in \allR})$. That is, $\aggr(\crit_a) \geq \aggr(\crit_b)$, satisfying efficiency.

\emph{Strategyproofness.}
Let $i$ be an arbitrary reviewer. Observe that in this setting, the aggregate score $\aggr(\crit_a)$ of a paper $a$ depends only on the score $y_{ia}$ and not on other scores $\{y_{ib}\}_{b \neq a}$ given by reviewer $i$. In other words, the only way to manipulate $\aggr(\crit_a) = \text{left-med}(\{y_{i'a}\}_{i' \in \allR})$ is by changing $y_{ia}$. Consider three cases. Suppose $y_{ia} < \aggr(\crit_a)$. In this case, if reviewer $i$ reports $y_{ia}' \leq \aggr(\crit_a)$, then there is no change in the aggregate score of $a$. On the other hand, if $y_{ia}' > \aggr(\crit_a)$, then either the aggregate score of $a$ remains the same or increases, making things only worse for reviewer $i$. The other case of $y_{ia} > \aggr(\crit_a)$ is symmetric to $y_{ia} < \aggr(\crit_a)$. Consider the third case, $y_{ia} = \aggr(\crit_a)$. In this case, manipulation can only make things worse since we already have $|y_{ia} - \aggr(\crit_a)| = 0$. In summary, reporting $y_{ia}'$ instead of $y_{ia}$ cannot help decrease $|y_{ia} - \aggr(\crit_a)|$. Also, recall that $y_{ia}$ does not affect the aggregate scores of other papers, and hence manipulation of $y_{ia}$ does not help them either. Therefore, by manipulating any of the $y_{ia}$ scores, reviewer $i$ cannot bring the aggregate recommendations closer to her own, proving strategyproofness.
\end{proof}

\subsection{Violation of the Axioms When $(p,q) \neq (1,1)$} 
We now tackle the harder `only if' direction of Theorem~\ref{thm:master}. We do so in three steps: efficiency is violated by $p\in(1,\infty)$ and $q=1$ (Lemma~\ref{lem:ssd-negative}), strategyproofness is violated by \Lpq{p}{q} aggregation for all $q>1$ (Lemma~\ref{lem:sp-negative}), and consensus is violated by $p=\infty$ and $q=1$ (Lemma~\ref{lem:consensus}). Together, the three lemmas leave $p=q=1$ as the only option. Below we state the lemmas and give some proof ideas; the theorem's full proof is relegated to Appendix~\ref{app:proofs}.

It is worth noting that, although we have presented the lemmas as components in the proof of Theorem~\ref{thm:master}, they also have standalone value (some more than others). For example, if one decided that only strategyproofness is important, then Lemma~\ref{lem:sp-negative} below would give significant guidance on choosing an appropriate method. 

\subsubsection{Violation of efficiency}\label{sec:efficiency_violated}

In our view, the following lemma presents the most interesting and counter-intuitive result in the paper. 
\begin{lemma} \label{lem:ssd-negative}
\Lpq{p}{q} aggregation with $p \in (1,\infty)$ and $q = 1$ violates efficiency.
\end{lemma}

It is quite surprising that such reasonable loss functions violate the  simple requirement of efficiency. In what follows we attempt to explain this phenomenon via a connection between our problem and the notion of the `Fermat point' of a triangle~\citep{spain1996fermat}. The explanation provided here demonstrates the negative result for \Lpq{2}{1} aggregation. The complete proof of the lemma for general values of $p \in (1,\infty)$ is quite involved, as can be seen in Appendix~\ref{app:proofs}. 

Consider a setting with $3$ reviewers and $2$ papers, where each reviewer reviews both papers. We let $\crit_1$ and  $\crit_2$ denote the respective objective criteria scores of the two papers. Assume that no score in $\{\crit_1,\crit_2\}$ is pointwise greater than or equal to the other score in that set. Let the overall recommendations given by the reviewers be $y_{11} = z$, $y_{21}=0$, $y_{31}=0$ to the first paper and $y_{12}=0$, $y_{22}=1$ and $y_{23}=0$ to the second paper. Under these scores, let $\aggr$ denote the aggregate function that minimizes the \Lpq{2}{1} loss. 

The Fermat point of a triangle is a point such that the sum of its (Euclidean) distances from all three vertices is minimized. Consider a triangle in $\mathbb{R}^2$ with vertices $(z,0)$, $(0,1)$ and $(0,0)$. Setting $z=2$, one can use known algorithms to compute the Fermat point of this triangle as $(0.25, 0.30)$. More generally, when the vertex $(z,0)$ is moved away from the rest of the triangle (by increasing $z$), the Fermat point paradoxically biases towards the other (second) coordinate. 

Connecting back to our original problem, by definition, the Fermat point of this triangle is exactly  $(\aggr(\crit_1),\aggr(\crit_2))$. When $z=2$, paper 1 receives scores $(2,0,0)$ in sorted order, which dominates the sorted scores $(1,0,0)$ of paper 2. However the aggregate score $\aggr(\crit_1) = 0.25$ of paper 1 is strictly smaller than $\aggr(\crit_2)=0.30$ of paper 2, thereby violating efficiency for the \Lpq{2}{1} loss.

As a final but important remark, the proof of Lemma~\ref{lem:ssd-negative} only requires a significantly weaker notion of efficiency. In this weaker notion, we consider two papers $a$ and $b$ such that their reviews are symmetric (formally, switching the labels $a$ and $b$ and switching the labels of some reviewers leaves the data unchanged). In this case, reducing the review scores of paper $b$ must lead to $\aggr(\crit_a)\geq \aggr(\crit_b)$. 

\subsubsection{Violation of strategyproofness}

\begin{lemma} \label{lem:sp-negative}
\Lpq{p}{q} aggregation with $q \in (1,\infty]$ violates strategyproofness.
\end{lemma}

We prove the lemma via a simple construction with just one paper and two reviewers, who give the paper overall recommendations of $1$ and $0$, respectively. For $q\in(1,\infty)$, the aggregate score is 
$$\aggr = \underset{f \in \mathbb{R}}{\text{argmin}}\ \Big\{|1 - f|^q + |f|^q\Big\},$$
and for $q=\infty$, it is
$$\aggr = \underset{f \in \mathbb{R}}{\text{argmin}}\  \max\big(|1-f|, |f|\big).$$
Either way, the unique minimum is obtained at an aggregate score of $0.5$. If reviewer 1 reported an overall recommendation of $2$, however, the aggregate score would be $1$, which matches her `true' recommendation, thereby violating strategyproofness. See Appendix~\ref{app:sp-negative} for the complete proof.

\subsubsection{Violation of consensus}

\begin{lemma} \label{lem:consensus}
\Lpq{p}{q} aggregation with $p = \infty$ and $q = 1$ violates consensus.
\end{lemma}

Lemma~\ref{lem:consensus} is established via another simple construction: two papers, two reviewers, and overall recommendations
\[\mathbf{y} = \begin{bmatrix}
0 & 1\\
2 & 1
\end{bmatrix},\]
where $y_{ia}$ denotes the overall recommendation given by reviewer $i$ to paper $a$. Crucially, the two reviewers agree on an overall recommendation of $1$ for paper $2$, hence the aggregate score of this paper must also be $1$. But we show that \Lpq{\infty}{1} aggregation would \emph{not} return an aggregate score of $1$ for paper $2$. The formal proof appears in Appendix~\ref{app:consensus}.


\section{Implementation and Experimental Results}
\label{sec:empirical}

\newcommand{\ijcaiaccept}{27.27\%}
\newcommand{\subrev}{k}
\newcommand{\optfn}{\widetilde{f}}

In this section, we provide an empirical analysis of a few aspects of peer review through the approach of this paper. We employ a dataset of reviews from the 26$^{\text{th}}$ International Joint Conference on Artificial Intelligence (IJCAI 2017), which was made available to us by the program chair. 
To our knowledge, we are the first to use this dataset.

At submission time, authors were asked if review data for their paper could be included in an anonymized dataset, and, similarly, reviewers were asked whether their reviews could be included; the dataset provided to us consists of all reviews for which permission was given. Each review is tagged with a reviewer ID and paper ID, which are anonymized for privacy reasons. The criteria used in the conference are `originality', `relevance', `significance', `quality of writing' (which we call `writing'), and `technical quality' (which we call `technical'), and each is rated on a scale from $1$ to $10$. Overall recommendations are also on a scale from $1$ to $10$. In addition, information about which papers were accepted and which were rejected is included in the dataset.

\begin{table*}[t]
\centering
\begin{tabular}{|c|ccllllccl|}
\hline
\# of reviews by a reviewer & 1 & 2 & 3 & 4 & 5 & 6 & 7 & 8 & $\geq 9$ \\ \hline
Frequency                   & 238 & 96 & 92 & 120 & 146 & 211 & 628 & 187 & 7 \\ \hline
\end{tabular}
\caption{Distribution of number of papers reviewed by a reviewer.}
\label{tab:revs-per-reviewer}
\end{table*}

The number of papers in the dataset is $2380$, of which $649$ were accepted, which amounts to $27.27\%$. This is a large subset of the $2540$ submissions to the conference, of which $660$ were accepted, for an actual acceptance rate of $25.98\%$. The number of reviewers in the dataset is $1725$, and the number of reviews is $9197$. All but nine papers in the dataset have three reviews ($485$ papers), four reviews ($1734$ papers), or five reviews ($152$) papers. Table~\ref{tab:revs-per-reviewer} shows the distribution of the number of papers reviewed by reviewers.

We apply \Lpq{1}{1} aggregation (i.e., $p=q=1$), as given in Equation~\eqref{eqn:aggregation}, to this dataset to learn the aggregate function. Let us denote that function by $\optfn$. The optimization problem in Equation~\eqref{eqn:aggregation} is convex, and standard optimization packages can efficiently compute the minimizer. Hence, importantly, computational complexity is a nonissue in terms of implementing our approach. 

Once we compute the aggregate function $\optfn$, we calculate the aggregate overall recommendation of each paper $a$ by taking the median of the aggregate reviewer scores for that paper obtained by applying $\optfn$ to the objective scores:
\begin{align}
\optOver(a) =\text{median}(\{ \optfn(\crit_{ia}) \}_{i \in R(a)}) \quad \forall a \in \allP. 
\end{align}
Recalling that $\ijcaiaccept$ of the papers in the dataset were actually accepted to the conference, in our experiments we define the set of papers accepted by the aggregate function $\optfn$ as the the top $27.27\%$ of papers according to their respective $\optOver$ values. 
We now present the specific experiments we ran, and their results.


\begin{figure}[ht]
\centering
\begin{minipage}{.47\textwidth}
  \centering
  \includegraphics[width=\textwidth]{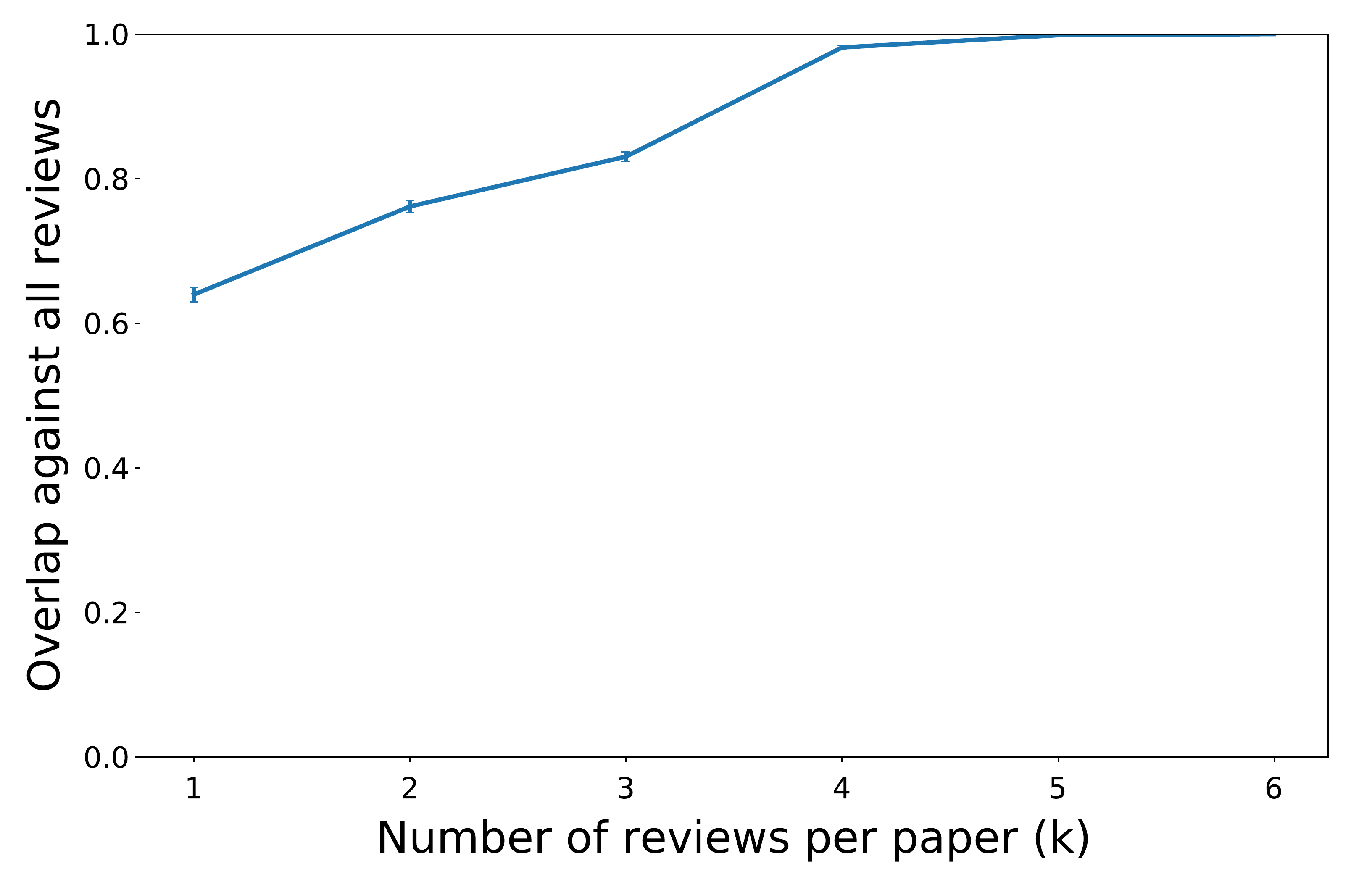}
  \caption{Fraction overlap as number of reviews per paper is restricted. Error bars depict $95\%$ confidence intervals, but may be too small to be visible for $k=4,5$.}
  \label{fig:revs-cutoff}
\end{minipage}\hfill
\begin{minipage}{.47\textwidth}
  \centering
  \includegraphics[width=\textwidth]{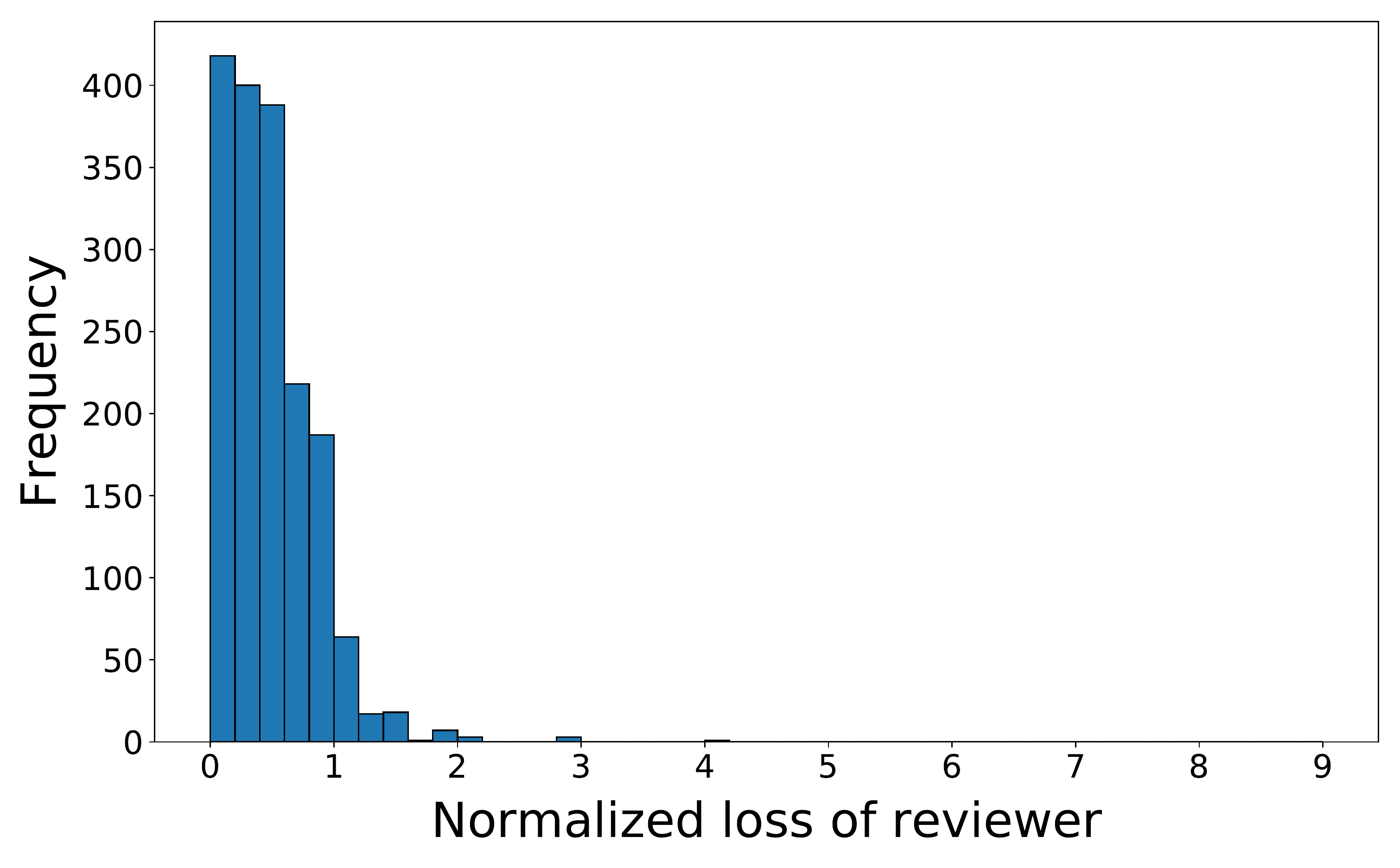}
  \caption{Frequency of losses of the reviewers for \Lpq{1}{1} aggregation, normalized by the number of papers reviewed by the respective reviewers.}
  \label{fig:loss-distribution}
\end{minipage}
\end{figure}

\subsection{Varying Number of Reviewers} 

In our first experiment, for each value of a parameter $\subrev \in \{1,\ldots,5\}$, we subsampled $\subrev$ distinct reviews for each paper uniformly at random from the set of all reviews for that paper  (if the paper had fewer than $\subrev$ to begin with then we retained all the reviews). We then computed an aggregate function, $\aggr_\subrev$, via \Lpq{1}{1} aggregation applied only to these subsampled reviews. Next, we found the set of top \ijcaiaccept{} papers as given by $\aggr_\subrev$ applied to the subsampled reviews. Finally, we compared the overlap of this set of top papers for every value of $\subrev$ with the set of top \ijcaiaccept{} papers as dictated by the overall aggregate function $\optfn$. 

The results from this experiment are plotted in Figure~\ref{fig:revs-cutoff}, and lead to several observations. First, the incremental overlap from $\subrev=4$ to $5$ is very small because there are very few papers that had $5$ or more reviews. Second, we see that the amount of overlap monotonically increases with the number of reviewers per paper $\subrev$, thereby serving as a sanity check on the data as well as our methods. Third, we observe the overlap to be quite high ($\approx 60\%$) even with a single reviewer per paper.


\subsection{Loss Per Reviewer}  
Next, we look at the loss of different reviewers, under $\optfn$ (obtained by \Lpq{1}{1} aggregation). In order for the losses to be on the same scale, we normalize each reviewer's loss by the number of papers reviewed by them. Formally, the normalized loss of reviewer $i$ (for $p=1$) is $$\frac{1}{|P(i)|}\sum_{a \in P(i)} |y_{ia} - \optfn(\crit_{ia})|.$$ The normalized loss averaged across reviewers is found to be $0.470$, and the standard deviation is $0.382$. Figure~\ref{fig:loss-distribution} shows the distribution of the normalized loss of all the reviewers. Note that the normalized loss of a reviewer can fall in the range $[0,9]$. These results thus indicate that the function $\optfn$ is indeed at least a reasonable representation of the mapping of the broader community.


\subsection{Overlap of Accepted Papers} 

We also compute the overlap between the set of top $27.27\%$ papers selected by \Lpq{1}{1} aggregation $\optfn$ with the \emph{actual} $27.27\%$  accepted papers. It is important to emphasize that we believe the set of papers selected by our method is \emph{better} than any hand-crafted or rule-based decision using the scores, since this aggregate represents the opinion of the community. Hence, to be clear, we do \emph{not} have a goal of maximizing the overlap. Nevertheless, a very small overlap would mean that our approach is drastically different from standard practice, which would potentially be disturbing. We find that the overlap is $79.2\%$, which we think is quite fascinating\,---\,our approach does make a significant difference, but the difference is not so drastic as to be disconcerting.

Out of intellectual curiosity, we also computed the pairwise overlaps of the papers accepted by \Lpq{p}{q} aggregation, for $p,q\in\{1,2,3\}$. We find that the choice of the reviewer-norm hyperparameter $q$ has more influence than the paper-norm hyperparameter $p$; we refer the reader to Appendix~\ref{app:influence} for details. Finally, in Appendix~\ref{app:visual} we present visualizations of \Lpq{1}{1} aggregation, which provide insights into the preferences of the community. 

\section{Discussion}

We address the problem of subjectivity in peer review by combining approaches from machine learning and social choice theory. A key challenge in the setting of peer review (e.g., when choosing a loss function) is the absence of ground truth, and we overcome this challenge via a principled, axiomatic approach.

One can think of the theoretical results of Section~\ref{sec:results} as supporting \Lpq{1}{1} aggregation using the tools of social choice theory, whereas the empirical results of Section~\ref{sec:empirical} focus on studying its \emph{behavior} on real data. Understanding this helps clear up another possible source of confusion: are we not overfitting by training on a set of reviews, and then applying the aggregate function to the same reviews? The answer is negative, because the process of learning the function $\aggr$ amounts to an aggregation of opinions about how criteria scores should be mapped to overall recommendations. Applying it to the data yields recommendations in $\domY$, whereas this function from $\domX$ to $\domY$ lives in a different space. 

That said, it is of intellectual interest to understand the statistical aspects of estimating the community's consensus mapping function, assuming the existence of a ground truth. In more detail, suppose that each reviewer's true function $\truef_i$ is a noisy version of some underlying function $f^{\star\star}$ that represents the community's beliefs. Then can \Lpq{1}{1} aggregation recover the function $f^{\star\star}$ (in the statistical consistency sense)? If so, then with what sample complexity? At a conceptual level, this non-parametric estimation problem is closely  related to problems in isotonic regression~\citep{SBGW16,gao2007entropy,chatterjee2018matrix}. The key difference is that the observations in our setting consist of evaluations of multiple functions, where each such function is a noisy version of the original monotonic function. In contrast, isotonic regression is primarily concerned with noisy evaluations of a common function. Nevertheless, the insights from isotonic regression suggest that the naturally occurring monotonicity assumption of our setting can yield attractive\,---\,and sometimes near-parametric~\citep{SBGW16,shah2017low}\,---\,rates of estimation.

Our work focuses on learning one representative aggregate mapping for the entire community of reviewers. Instead, the program chairs of a conference may wish to allow for multiple mappings that represent the aggregate opinions of different sub-communities (e.g., theoretical or applied researchers). In this case, one can modify our framework to also learn this (unknown) partition of reviewers and/or papers into multiple sub-communities with different mapping functions, and frame the problem in terms of learning a mixture model. The design of computationally efficient algorithms for \Lpq{p}{q} aggregation under such a mixture model is a challenging open problem. 

As a final remark, we see our work as an unusual synthesis between computational social choice and machine learning.  We hope that our approach will inspire exploration of additional connections between these two fields of research, especially in terms of viewing choices made in machine learning\,---\,often in an \emph{ad hoc} fashion\,---\,through the lens of computational social choice.

\section*{Acknowledgments}
Shah was supported in part by NSF grants CRII-CCF-1755656 and CCF-1763734. Noothigattu and Procaccia were supported in part by NSF grants IIS-1350598, IIS-1714140, CCF-1525932, and CCF-1733556; by ONR grants
N00014-16-1-3075 and N00014-17-1-2428; and by a Sloan Research Fellowship and a Guggenheim Fellowship. We are grateful to Francisco Cruz for compiling the IJCAI 2017 review dataset, and to Carles Sierra for making it available to us.

\bibliography{abb,ultimate}

\appendix

\section{Proof Of Theorem~\ref{thm:master}}
\label{app:proofs}

Recall that the proof of our main result,  Theorem~\ref{thm:master}, includes four lemmas. Here we prove the three lemmas whose proofs were omitted from the main text.

\subsection{Proof of Lemma~\ref{lem:ssd-negative}} \label{app:ssd}

Consider $L(p,1)$ aggregation with an arbitrary $p \in (1, \infty)$. We show that efficiency is violated using the following construction. 
There are $2$ papers, $3$ reviewers and each reviewer reviews both papers. Assume that the papers have objective criteria scores $\crit_1$ and $\crit_2$, and that neither of these scores is pointwise greater than or equal to the other. Let the overall recommendations by the reviewers for the papers be defined by the matrix
\[\mathbf{y} = \begin{bmatrix}
z & 0\\
0 & 1\\
0 & 0
\end{bmatrix},\]
where $z$ is a constant strictly bigger than $1$ and $y_{ia}$ denotes the overall recommendation by reviewer $i$ to paper $a$. Observe that paper $1$ dominates paper $2$. But, we will show that there exists a value $z > 1$ such that the aggregate score of paper $1$ is strictly smaller than the aggregate score of paper $2$.

Let $f_i$ denote the value of function $f$ on paper $i$, i.e. $f_i := f(\crit_i)$. And let $\aggr_i(z)$ denote the aggregate score of paper $i$; observe that we write it as a function of $z$ because the aggregate score of each paper would depend on the chosen score $z$. Since we are minimizing \Lpq{p}{1} loss, the aggregate function satisfies:
\begin{align}	\label{eqn:ssd-opt}
(\aggr_1(z), \aggr_2(z)) \in \underset{(f_1, f_2) \in \mathbb{R}^2}{\text{argmin}} \bigg\{\big\|(z,0) - (f_1, f_2)\big\|_p + \big\|(0,1) - (f_1, f_2)\big\|_p +  \big\|(f_1, f_2)\big\|_p \bigg\}.
\end{align}
We do not have any monotonicity constraints in~\eqref{eqn:ssd-opt} because the two papers have incomparable criteria scores. For simplicity, let $\vec{f} := (f_1, f_2)$, $\vec{\aggr}(z) := (\aggr_1(z), \aggr_2(z))$, and denote the objective function in Equation~\eqref{eqn:ssd-opt} by $G_z(\vec{f})$. That is,
\begin{align} \label{eqn:ssd-opt2}
G_z(f_1,f_2) &= \Big[|z-f_1|^p + |f_2|^p\Big]^{\frac{1}{p}} + \Big[|f_1|^p + |1-f_2|^p \Big]^{\frac{1}{p}} + \Big[|f_1|^p + |f_2|^p\Big]^{\frac{1}{p}}.
\end{align}
For the overall proof to be easier to follow, proofs of all claims are given at the end of this proof. Also, just to re-emphasize, the whole proof assumes $z > 1$.

\begin{claim}   \label{claim:strictly-convex}
$G_z$ is a strictly convex objective function.
\end{claim}

Claim~\ref{claim:strictly-convex} states that $G_z$ is strictly convex, implying that it has a unique minimizer $\vec{\aggr}(z)$. Hence, there is no need to consider tie-breaking.

\begin{claim}   \label{claim:ssd-bounds}
$\aggr_1(z)$ and $\aggr_2(z)$ are bounded. In particular, $\aggr_1(z) \in [0,1]$ and $\aggr_2(z) \in [0,1]$.
\end{claim}

Claim~\ref{claim:ssd-bounds} states that the aggregate score of both papers lies in the interval $[0,1]$ irrespective of the value of $z$. This allow us to restrict ourselves to the region $[0,1]^2$ when computing the minimizer of~\eqref{eqn:ssd-opt2}. Hence, for the rest of the proof, we only consider the space $[0,1]^2$. In this region, the optimization problem~\eqref{eqn:ssd-opt} can be rewritten as
\begin{align*}
(\aggr_1(z), \aggr_2(z)) = \underset{f_1 \in [0,1], f_2 \in [0,1]}{\text{argmin}} \bigg\{\Big[\big(z-f_1\big)^p + f_2^p\Big]^{\frac{1}{p}} + \Big[f_1^p + \big(1-f_2\big)^p\Big]^{\frac{1}{p}} + \Big[f_1^p + f_2^p\Big]^{\frac{1}{p}} \bigg\}.
\end{align*}
To start off, we analyze the objective function as we take the limit of $z$ going to infinity. Later, we show that the observed property holds even for a sufficiently large finite $z$.

For the limit to exist, redefine the objective function as $H_z(f_1, f_2) = G_z(f_1, f_2) - G_z(0,0)$, i.e.,
\begin{align}   \label{eqn:ssd-opt3}
H_z(f_1, f_2) = \Big[\big(z-f_1\big)^p + f_2^p\Big]^{\frac{1}{p}} -z + \Big[f_1^p + \big(1-f_2\big)^p \Big]^{\frac{1}{p}} + \Big[f_1^p + f_2^p\Big]^{\frac{1}{p}} - 1.
\end{align}
For any value of $z$, the function $H_z$ has the same minimizer as $G_z$, that is,
\begin{align*}
(\aggr_1(z), \aggr_2(z)) = \underset{f_1 \in [0,1], f_2 \in [0,1]}{\text{argmin}} H_z(f_1, f_2).
\end{align*}

\begin{claim}   \label{claim:ssd-limit}
For any (fixed) $f_1 \in [0,1], f_2 \in [0,1]$,
$$\lim_{z \to \infty} H_z(f_1, f_2) = H^\star(f_1, f_2),$$
where
\begin{align}   \label{eqn:ssd-limit}
H^\star(f_1, f_2) = -f_1 + \Big[f_1^p + \big(1-f_2\big)^p\Big]^{\frac{1}{p}} + \Big[f_1^p + f_2^p\Big]^{\frac{1}{p}} - 1.
\end{align}
\end{claim}

The proof proceeds by analyzing some important properties of the limiting function $H^\star$.

\begin{claim}   \label{claim:limit-convexity}
The function $H^\star(\vec f)$ is convex in $\vec f \in [0,1]^2$. Moreover, the function $H^\star(\vec f)$ is strictly convex for $f_1 \in (0,1]$ and $f_2 \in [0,1]$.
\end{claim}

\begin{claim}   \label{claim:limit-mins}
$H^\star$ is minimized at $\vec{\infopt} = (\infopt_1, \infopt_2)$, where
\begin{align}   \label{eqn:limit-mins}
\infopt_1 = \frac{1}{2} \left[\frac{1}{(2^{\frac{p}{p-1}}-1)}\right]^{\frac{1}{p}}, \qquad \infopt_2 = \frac{1}{2}.
\end{align}
\end{claim}

\begin{claim}   \label{claim:ssd-inf-order}
$\infopt_1 < \infopt_2$.
\end{claim}

Observe that Claim~\ref{claim:ssd-inf-order} is the desired result, but for the limiting objective function $H^\star$. The remainder of the proof proceeds to show that this result holds even for the objective function $H_z$, when the score $z$ is large enough. Define $\Delta = \infopt_2 - \infopt_1 > 0$. We first show that (i) there exists $z > 1$ such that $\|\vec{\aggr}(z) - \vec{\infopt}\|_2 < \frac{\Delta}{4}$, and then (ii) show that in this case, we have $\aggr_1(z) < \aggr_2(z)$.

To prove part (i), we first analyze how functions $H_z$ and $H^\star$ relate to each other. Using Claim~\ref{claim:ssd-limit}, for any fixed $f_1, f_2$, by definition of the limit, for any $\epsilon > 0$, there exists $z_\epsilon$ (which could be a function of $f_1, f_2$) such that, for all $z > z_\epsilon$, we have
\begin{align}   \label{eqn:pointwise-limit}
|H_z(f_1, f_2) - H^\star(f_1, f_2)| < \epsilon.
\end{align}
For a given $f_1, f_2$, denote the corresponding value of $z_\epsilon$ by $z_\epsilon(f_1, f_2)$. And, let $\mathcal{Z}_\epsilon(f_1, f_2)$ denote the set of all values of $z > 1$ for which Equation~\eqref{eqn:pointwise-limit} holds for $(f_1, f_2)$.

\begin{claim}   \label{claim:z-eps-order}
$\mathcal{Z}_\epsilon(1, 1) \subset \mathcal{Z}_\epsilon(f_1, f_2)$ for every $(f_1, f_2) \in [0,1]^2$.
\end{claim}

Claim~\ref{claim:z-eps-order} says that if Equation~\eqref{eqn:pointwise-limit} holds for a particular value of $z$ for $f_1 = f_2 = 1$, then for the same value of $z$ it holds for every other value of $(f_1, f_2) \in [0,1]^2$ as well. So, define
\begin{align}   \label{eqn:max-z}
\widetilde{z}_\epsilon := z_\epsilon(1, 1) + 1.
\end{align}
By definition, $\widetilde{z}_\epsilon \in \mathcal{Z}_\epsilon(1,1)$. And by Claim~\ref{claim:z-eps-order}, $\widetilde{z}_\epsilon \in \mathcal{Z}_\epsilon(f_1, f_2)$ for every $(f_1, f_2) \in [0,1]^2$.
So, set $z = \widetilde{z}_\epsilon$. Then, Equation~\eqref{eqn:pointwise-limit} holds for all $(f_1, f_2) \in [0,1]^2$ simultaneously. In other words, for all $(f_1, f_2) \in [0,1]^2$, we simultaneously have
\begin{align}   \label{eqn:epsilon-band}
H^\star(f_1, f_2) - \epsilon < H_z(f_1, f_2) < H^\star(f_1, f_2) + \epsilon,
\end{align}
i.e. $H_z$ is in an $\epsilon$-band around $H^\star$ throughout this region. And observe that this band gets smaller as $\epsilon$ is decreased (which is achieved at a larger value of $z$).

To bound the distance between $\vec \infopt$, the minimizer of $H^\star$, and $\vec \aggr(z)$, the minimizer of $H_z$, we bound the distance between the objective function values at these points.

\begin{claim}   \label{claim:fn-value-diff}
$H^\star(\vec{\aggr}(z)) < H^\star(\vec{\infopt}) + 2 \epsilon.$
\end{claim}

Although $\vec{\aggr}(z)$ does not minimize $H^\star$, Claim~\ref{claim:fn-value-diff} says that the objective value at $\vec{\aggr}(z)$ cannot be more than $2 \epsilon$ larger than its minimum, $H^\star(\vec \infopt)$. We use this to bound the distance between $\vec \aggr(z)$ and the minimizer $\vec \infopt$. Observe that $\vec \aggr(z)$ falls in the $[H^\star(\vec \infopt) + 2\epsilon]$-level set of $H^\star$. So, we next look at a specific level set of $H^\star$.

Define
\begin{align}   \label{eqn:tau-def}
\tau := \min_{\vec f \in [0,1]^2 : \|\vec f - \vec \infopt\|_2 = \frac{\Delta}{4}} H^\star(\vec f).
\end{align}
Observe that a minimum exists (infimum is not required) for the minimization in~\eqref{eqn:tau-def} because we are minimizing over the closed set $\{\vec f \in [0,1]^2 : \|\vec f - \vec \infopt\|_2 = \frac{\Delta}{4}\}$ and $H^\star$ is continuous.

For any fixed $p \in (1,\infty)$, Equation~\eqref{eqn:limit-mins} shows that $\infopt_1$ is bounded away from $0$. Hence, Claim~\ref{claim:limit-convexity} shows that $H^\star$ is strictly convex at and in the region around $\vec \infopt$. Further, $H^\star$ is convex everywhere else. Coupling this with the fact that~\eqref{eqn:tau-def} minimizes along points not arbitrarily close to the minimizer $\vec \infopt$, we have $\tau > H^\star(\vec \infopt)$.

Define the level set of $H^\star$ with respect to $\tau$:
\begin{align*}
    \mathcal{C}_\tau = \{\vec f \in [0,1]^2: H^\star(\vec f) \leq \tau\}.
\end{align*}

\begin{claim}   \label{claim:level-set-prop}
For every $\vec f \in \mathcal{C}_\tau$, we have 
$\|\vec f - \vec \infopt\|_2 \leq \frac{\Delta}{4}.$
\end{claim}

Define $\epsilon_o := \frac{\tau - H^\star(\vec \infopt)}{2}$, and set $\epsilon = \epsilon_o$. Then, set $z = \widetilde{z}_{\epsilon_o}$ as before. Applying Claim~\ref{claim:fn-value-diff}, we obtain
$$H^\star(\vec \aggr(\widetilde{z}_{\epsilon_o})) < H^\star(\vec \infopt) + 2 \epsilon_o = \tau.$$
In other words, $\vec \aggr(\widetilde{z}_{\epsilon_o}) \in \mathcal{C}_\tau$. And applying Claim~\ref{claim:level-set-prop}, we obtain $\|\vec \aggr(\widetilde{z}_{\epsilon_o}) - \vec  \infopt\|_2 \leq \frac{\Delta}{4}$, completing part (i).

This implies that $\|\vec \aggr(\widetilde{z}_{\epsilon_o}) - \vec  \infopt\|_\infty \leq \frac{\Delta}{4}$, which means
\begin{align}   \label{eqn:comp-distances}
\left|\aggr_1(\widetilde{z}_{\epsilon_o}) -  \infopt_1\right| \leq \frac{\Delta}{4} \quad \text{and} \quad \left|\aggr_2(\widetilde{z}_{\epsilon_o}) - \infopt_2\right| \leq \frac{\Delta}{4}.
\end{align}

Using these properties, we have 
\begin{align*}
\aggr_1(\widetilde{z}_{\epsilon_o}) &\leq \infopt_1 + \frac{\Delta}{4}\\
&= \infopt_2 - \Delta + \frac{\Delta}{4}\\
&\leq \aggr_2(\widetilde{z}_{\epsilon_o}) + \frac{\Delta}{4} - \Delta + \frac{\Delta}{4} = \aggr_2(\widetilde{z}_{\epsilon_o}) - \frac{\Delta}{2},
\end{align*}
where the first inequality holds because of the first part of~\eqref{eqn:comp-distances}, the equality holds because $\Delta = \infopt_2 - \infopt_1$ and the second inequality holds because of the second part of~\eqref{eqn:comp-distances}. Therefore, for $z = \widetilde{z}_{\epsilon_o} > 1$, the aggregate scores of the two papers are such that
$$\aggr_1(\widetilde{z}_{\epsilon_o}) < \aggr_2(\widetilde{z}_{\epsilon_o}),$$
violating efficiency.

\medskip
\noindent\textbf{Proof of Claim~\ref{claim:strictly-convex}}\hspace{5px}
Take arbitrary $\vec{f}, \vec{g} \in \mathbb{R}^2$ with $
\vec{f} \neq \vec{g}$, and let $\theta \in (0,1)$. We show that $G_z(\theta \vec{f} + (1-\theta)\vec{g}) < \theta G_z(\vec{f}) + (1-\theta)G_z(\vec{g})$. For this, we will first show that either (i) $[(z,0) - \vec{f}]$ is not parallel to $[(z,0) - \vec{g}]$, (ii) $[(0,1) - \vec{f}]$ is not parallel to $[(0,1) - \vec{g}]$ or (iii) $\vec{f}$ is not parallel to $\vec{g}$. For the sake of contradiction, assume that this is not true. That is, assume $[(z,0) - \vec{f}]$ is parallel to $[(z,0) - \vec{g}]$, $[(0,1) - \vec{f}]$ is parallel to $[(0,1) - \vec{g}]$, and $\vec{f}$ is parallel to $\vec{g}$. This implies that
\[\begin{bmatrix} z-f_1 \\ -f_2 \end{bmatrix} = r \begin{bmatrix} z-g_1 \\ -g_2 \end{bmatrix},\qquad
\begin{bmatrix} -f_1 \\ 1-f_2 \end{bmatrix} = s \begin{bmatrix} -g_1 \\ 1-g_2 \end{bmatrix} \quad \text{and} \quad \begin{bmatrix} f_1 \\ f_2 \end{bmatrix} = t \begin{bmatrix} g_1 \\ g_2 \end{bmatrix},\]
where $r,s,t \in \mathbb{R}$ \footnote{A boundary case not captured here is when $
\vec{g}$ is exactly one of the points $(z,0), (0,1)$ or $(0,0)$, leading to $1/r, 1/s$ or $1/t$ being zero respectively. But for this case, it is easy to prove that the other two pairs of vectors cannot be parallel unless $\vec{f} = \vec{g}$.}. Note that, none of $r,s,t$ can be $1$ because $\vec{f} \neq \vec{g}$. The second equation tells us that $f_1 = s g_1$ and the third one tells us that $f_1 = t g_1$. So, either $f_1 = g_1 = 0$ or $s = t$. But from the first equation, $z-f_1 = r z - r g_1$. So if $f_1 = g_1 = 0$, it says that $r = 1$ which is not possible. Therefore, $s = t$. The third equation now tells us that $f_2 = t g_2 = s g_2$. But, the second equation gives us $1-f_2 = s - s g_2$, which implies that $s = 1$. But again this is not possible, leading to a contradiction. Therefore, at least one of (i), (ii) and (iii) is true.

$L_p$ norm with $p \in (1, \infty)$ is a convex norm, i.e. for any $x, y \in \mathbb{R}^2$,
\begin{align}   \label{eqn:convexity}
\|\theta x + (1-\theta)y\|_p \leq \theta \|x\|_p + (1-\theta)\|y\|_p.
\end{align}
Further, since $p \in (1, \infty)$, the inequality in~\eqref{eqn:convexity} is strict if $x$ is not parallel to $y$. For our objective (in Equation~\eqref{eqn:ssd-opt}),
\begin{align}
G_z(\theta \vec{f} + (1-\theta)\vec{g}) = \big\|&\theta[(z,0) - \vec{f}] + (1-\theta)[(z,0) - \vec{g}]\big\|_p \notag\\
&+ \big\|\theta[(0,1) - \vec{f}] + (1-\theta)[(0,1) - \vec{g}]\big\|_p \notag\\
&+ \big\|\theta \vec{f} + (1-\theta) \vec{g}\big\|_p. \label{eqn:G-convexity}
\end{align}
Because of convexity of the $L_p$ norm, each of the three terms on the RHS of Equation~\eqref{eqn:G-convexity} satisfies inequality~\eqref{eqn:convexity}. Further, because at least one of the pair of vectors in the three terms is not parallel (since either (i), (ii) or (iii) is true), at least one of them gives us a strict inequality. Therefore we obtain
\[G_z(\theta \vec{f} + (1-\theta)\vec{g}) < \theta G_z(\vec{f}) + (1-\theta)G_z(\vec{g}).\]
\hfill$\blacksquare$

\medskip
\noindent\textbf{Proof of Claim~\ref{claim:ssd-bounds}}\hspace{5px}
The claim has four parts: (i) $\aggr_1(z) \geq 0$, (ii) $\aggr_1(z) \leq 1$, (iii) $\aggr_2(z) \geq 0$ and (iv) $\aggr_2(z) \leq 1$. Observe that parts (i), (iii) and (iv) are more intuitive, since they show that the aggregate score of a paper is no higher than the maximum score given to it, and no lower than the minimum score given to it. Part (ii) on the other hand is stronger; even though paper $1$ has a score of $z > 1$ given to it, this part shows that $\aggr_1(z) \leq 1$ (which is much tighter than an upper bound of $z$, especially when $z$ is large). We prove the simpler parts (i), (iii) and (iv) first.

For the sake of contradiction, suppose $\aggr_1(z) < 0$. Then
\begin{small}
\begin{align*}
G_z(\aggr_1(z),\aggr_2(z)) &= \Big[|z-\aggr_1(z)|^p + |\aggr_2(z)|^p\Big]^{\frac{1}{p}} + \Big[|\aggr_1(z)|^p + |1-\aggr_2(z)|^p\Big]^{\frac{1}{p}} + \Big[|\aggr_1(z)|^p + |\aggr_2(z)|^p\Big]^{\frac{1}{p}}\\
&> \Big[|z|^p + |\aggr_2(z)|^p\Big]^{\frac{1}{p}} + \Big[0 + |1-\aggr_2(z)|^p\Big]^{\frac{1}{p}} + \Big[0 + |\aggr_2(z)|^p\Big]^{\frac{1}{p}} = G_z(0, \aggr_2(z)),
\end{align*}
\end{small}
{\flushleft contradicting the fact that $(\aggr_1(z), \aggr_2(z))$ is optimal. Therefore, $\aggr_1(z) \geq 0$, completing proof of (i). Similarly, if $\aggr_2(z) < 0$, we can show that $G_z(\aggr_1(z), \aggr_2(z)) > G_z(\aggr_1(z), 0)$, violating optimality. Therefore, $\aggr_2(z) \geq 0$, completing proof of (iii).}

Next, for the sake of contradiction assume that $\aggr_2(z) > 1$. Then
\begin{small}
\begin{align*}
G_z(\aggr_1(z),\aggr_2(z)) &= \Big[|z-\aggr_1(z)|^p + |\aggr_2(z)|^p\Big]^{\frac{1}{p}} + \Big[|\aggr_1(z)|^p + |1-\aggr_2(z)|^p\Big]^{\frac{1}{p}} + \Big[|\aggr_1(z)|^p + |\aggr_2(z)|^p\Big]^{\frac{1}{p}}\\
&> \Big[|z-\aggr_1(z)|^p + 1\Big]^{\frac{1}{p}} + \Big[|\aggr_1(z)|^p + 0\Big]^{\frac{1}{p}} + \Big[|\aggr_1(z)|^p + 1\Big]^{\frac{1}{p}} = G_z(\aggr_1(z), 1),
\end{align*}
\end{small}
{\flushleft contradicting the fact that $(\aggr_1(z), \aggr_2(z))$ is optimal. Therefore, we also have $\aggr_2(z) \leq 1$, completing proof of (iv).}

Finally, we prove the more non-intuitive part, (ii). Suppose for the sake of contradiction, $\aggr_1(z) > 1$. Then,
\begin{small}
\begin{align*}
G_z(\aggr_1(z),\aggr_2(z)) &= \Big[|z-\aggr_1(z)|^p + |\aggr_2(z)|^p\Big]^{\frac{1}{p}} + \Big[|\aggr_1(z)|^p + |1-\aggr_2(z)|^p\Big]^{\frac{1}{p}} + \Big[|\aggr_1(z)|^p + |\aggr_2(z)|^p\Big]^{\frac{1}{p}}\\
&\geq |z - \aggr_1(z)| + |\aggr_1(z)| + |\aggr_1(z)|\\
&\geq z + |\aggr_1(z)|,
\end{align*}
\end{small}
{\flushleft where the first inequality comes from the fact that the $L_p$ norm of each vector is at least as high as the absolute value of its first element, and the second inequality follows from the triangle inequality. Using the assumption that $\aggr_1(z) > 1$, we obtain
\[G_z(\aggr_1(z),\aggr_2(z)) > z+1 = G_z(0,0),\]
contradicting the fact that $(\aggr_1(z), \aggr_2(z))$ is optimal. Therefore, $\aggr_1(z) \leq 1$, completing the proof.}
\hfill$\blacksquare$

\medskip
\noindent\textbf{Proof of Claim~\ref{claim:ssd-limit}}\hspace{5px}
Take any arbitrary $f_1 \in [0,1]$ and $f_2 \in [0,1]$. Subtracting Equations~\eqref{eqn:ssd-opt3} and~\eqref{eqn:ssd-limit} we obtain
\begin{align}   \label{eqn:H-difference}
    H_z(f_1, f_2) - H^\star(f_1, f_2) = \Big[\big(z-f_1\big)^p + f_2^p\Big]^{\frac{1}{p}} - \big(z - f_1\big).
\end{align}
Observe that since $f_2 \geq 0$, the RHS of Equation~\eqref{eqn:H-difference} is non-negative. Hence, the equation does not change on using an absolute value, i.e.,
\begin{align}   \label{eqn:H-abs-difference}
|H_z(f_1, f_2) - H^\star(f_1, f_2)| = \Big[\big(z-f_1\big)^p + f_2^p\Big]^{\frac{1}{p}} - \big(z - f_1\big).
\end{align}
To prove the required result, we take a small detour and define $\phi(x) = \left(x^p + f_2^p\right)^{\frac{1}{p}} - x$. We show that $\phi(x) \to 0$ as $x \to \infty$. For this, rewrite $\phi(x)$ as follows
\begin{align*}
\phi(x) = x \left(1 + \frac{f_2^p}{x^p}\right)^{\frac{1}{p}} - x = \frac{\left(1 + \frac{f_2^p}{x^p}\right)^{\frac{1}{p}} - 1}{\frac{1}{x}}.
\end{align*}
Taking the limit of $x$ to infinity, we have
\begin{align}\label{eqn:phi-limit}
\lim_{x\to\infty}\phi(x) = \lim_{x\to\infty}\frac{\left(1 + \frac{f_2^p}{x^p}\right)^{\frac{1}{p}} - 1}{\frac{1}{x}}.
\end{align}
Observe that for both the numerator and denominator in the RHS of Equation~\eqref{eqn:phi-limit}, we have
$$\lim_{x \to \infty} \left\{\left(1 + \frac{f_2^p}{x^p}\right)^{\frac{1}{p}} - 1\right\} = 0
\quad \text{and} \quad
\lim_{x\to\infty}\left\{\frac{1}{x}\right\} = 0.$$
Hence, applying L'Hospital's rule on equation~\eqref{eqn:phi-limit} gives us
\begin{align*}
\lim_{x\to\infty}\phi(x) &= \lim_{x\to\infty}
\frac{-\frac{f_2^p}{x^{p+1}}\left(1 + \frac{f_2^p}{x^p}\right)^{\frac{1}{p}-1}}{-\frac{1}{x^2}}\\
&= \lim_{x\to\infty}
\left\{\frac{f_2^p}{x^{p-1}}\left(1 + \frac{f_2^p}{x^p}\right)^{\frac{1}{p}-1}\right\}\\
&= \left[\lim_{x\to\infty}
\frac{f_2^p}{x^{p-1}}\right] * \left[\lim_{x\to\infty}\left(1 + \frac{f_2^p}{x^p}\right)^{\frac{1}{p}-1}\right]\\
&= 0 * 1 = 0,
\end{align*}
where $\left[\lim_{x\to\infty} \frac{f_2^p}{x^{p-1}}\right] = 0$ because $p > 1$. Hence, we proved the required result, $\lim_{x \to \infty} \phi(x) = 0$. Going back to Equation~\eqref{eqn:H-abs-difference}, we rewrite it as
\begin{align*}
|H_z(f_1, f_2) - H^\star(f_1, f_2)| = \Big[\big(z-f_1\big)^p + f_2^p\Big]^{\frac{1}{p}} - \big(z - f_1\big) = \phi(z - f_1).
\end{align*}
Taking the limit of $z$ to infinity, we obtain
\begin{align}\label{eqn:diff-limit}
\lim_{z\to\infty} |H_z(f_1, f_2) - H^\star(f_1, f_2)| = \lim_{z\to\infty} \phi(z - f_1) = \lim_{t\to\infty} \phi(t) = 0,
\end{align}
where the second step follows by setting $t = z - f_1$. Equation~\eqref{eqn:diff-limit} implies that
$$\lim_{z\to\infty}  H_z(f_1, f_2) = H^\star(f_1, f_2).$$
\hfill$\blacksquare$

\medskip
\noindent\textbf{Proof of Claim~\ref{claim:limit-convexity}}\hspace{5px}
In the region $[0,1]^2$, using~\eqref{eqn:ssd-limit}, the function $H^\star$ can be written as
\begin{align}   \label{eqn:limit-rewritten}
H^\star(f_1, f_2) = -f_1 + \left\|(0,1) - (f_1, f_2)\right\|_p + \|(f_1, f_2)\|_p - 1.
\end{align}
Observe that each term on the RHS of~\eqref{eqn:limit-rewritten} is a convex function of $\vec f$. Hence, their sum is also convex in $\vec f$.

The proof of strict convexity closely follows the proof of claim~\ref{claim:strictly-convex}. Take arbitrary $\vec{f}, \vec{g} \in (0,1] \times [0,1]$ with $\vec{f} \neq \vec{g}$, and let $\theta \in (0,1)$. We show that $H^\star(\theta \vec{f} + (1-\theta)\vec{g}) < \theta H^\star(\vec{f}) + (1-\theta)H^\star(\vec{g})$. For this, we will first show that either (i) $[(0,1) - \vec{f}]$ is not parallel to $[(0,1) - \vec{g}]$ or (ii) $\vec{f}$ is not parallel to $\vec{g}$. For the sake of contradiction, assume that this is not true. That is, assume $[(0,1) - \vec{f}]$ is parallel to $[(0,1) - \vec{g}]$, and $\vec{f}$ is parallel to $\vec{g}$. This implies that
\[\begin{bmatrix} -f_1 \\ 1-f_2 \end{bmatrix} = r \begin{bmatrix} -g_1 \\ 1-g_2 \end{bmatrix} \quad \text{and} \quad \begin{bmatrix} f_1 \\ f_2 \end{bmatrix} = s \begin{bmatrix} g_1 \\ g_2 \end{bmatrix},\]
where $r,s \in \mathbb{R}$. Note that, neither $r$ nor $s$ can be $1$ because $\vec{f} \neq \vec{g}$. The first equation tells us that $f_1 = r g_1$ and the second one tells us that $f_1 = s g_1$. And since $g_1 \neq 0$, this implies that $r = s$. The second part of the second equation now tells us that $f_2 = s g_2 = r g_2$. The second part of the first equation becomes $1 - f_2 = r - rg_2$ which implies that $r=1$, leading to a contradiction. Therefore, at least one of (i) and (ii) is true.

Recall, $L_p$ norm with $p \in (1, \infty)$ is a convex norm, i.e. for any $x, y \in \mathbb{R}^2$,
\begin{align}   \label{eqn:convexity-again}
\|\theta x + (1-\theta)y\|_p \leq \theta \|x\|_p + (1-\theta)\|y\|_p.
\end{align}
And since $p \in (1, \infty)$, the inequality in~\eqref{eqn:convexity-again} is strict if $x$ is not parallel to $y$. For $H^\star$ (using Equation~\eqref{eqn:limit-rewritten}),
\begin{align}
H^\star(\theta \vec{f} + (1-\theta)\vec{g}) = &- \theta f_1 - (1-\theta) g_1 \notag\\
&+ \big\|\theta[(0,1) - \vec{f}] + (1-\theta)[(0,1) - \vec{g}]\big\|_p \notag\\
&+ \big\|\theta \vec{f} + (1-\theta) \vec{g}\big\|_p - 1. \label{eqn:H-convexity}
\end{align}
Because of convexity of the $L_p$ norm, both the third and fourth term on the RHS of Equation~\eqref{eqn:H-convexity} satisfy inequality~\eqref{eqn:convexity-again}. Further, because at least one of the pair of vectors in these two terms is not parallel (since either (i) or (ii) is true), at least one of them gives us a strict inequality. Therefore we obtain
\[H^\star(\theta \vec{f} + (1-\theta)\vec{g}) < \theta H^\star(\vec{f}) + (1-\theta)H^\star(\vec{g}).\]
\hfill$\blacksquare$

\medskip
\noindent\textbf{Proof of Claim~\ref{claim:limit-mins}}\hspace{5px}
To compute the minimizer of $H^\star$, we compute its gradients with respect to $f_1$ and $f_2$. Using Equation~\eqref{eqn:ssd-limit}, the partial derivative with respect to $f_1$ is
\begin{align}   \label{eqn:ssd-f1-der}
\frac{\partial H^\star}{\partial f_1} = -1 + f_1^{p-1} \Big[f_1^p + (1-f_2)^p\Big]^{\frac{1}{p}-1} + f_1^{p-1} \Big[f_1^p + f_2^p\Big]^{\frac{1}{p}-1}
\end{align}
and with respect to $f_2$ is
\begin{align}   \label{eqn:ssd-f2-der}
\frac{\partial H^\star}{\partial f_2} = 0 -(1-f_2)^{p-1}\Big[f_1^p + (1-f_2)^p\Big]^{\frac{1}{p}-1} +  f_2^{p-1} \Big[f_1^p + f_2^p\Big]^{\frac{1}{p}-1}.
\end{align}
Observe that at $f_2 = \frac{1}{2}$, irrespective of the value of $f_1$, the partial derivative~\eqref{eqn:ssd-f2-der} is
\begin{align*}
\frac{\partial H^\star}{\partial f_2}\bigg|_{f_2 = \frac{1}{2}} = -\frac{1}{2^{p-1}}\left[f_1^p + \frac{1}{2^p}\right]^{\frac{1}{p}-1} +  \frac{1}{2^{p-1}} \left[f_1^p + \frac{1}{2^p}\right]^{\frac{1}{p}-1} = 0.
\end{align*}
So, set $\infopt_2 = \frac{1}{2}$. Next, we find $\infopt_1$ such that the other derivative~\eqref{eqn:ssd-f1-der} is also zero at $\vec \infopt = (\infopt_1, \infopt_2)$. Setting~\eqref{eqn:ssd-f1-der} to zero at $\vec \infopt$, we obtain
\begin{align*}
\frac{\partial H^\star}{\partial f_1}\bigg|_{\vec f = \vec \infopt} = 0 &= -1 + \infopt_1^{p-1} \left[\infopt_1^p + \frac{1}{2^p}\right]^{\frac{1}{p}-1} + \infopt_1^{p-1} \left[\infopt_1^p + \frac{1}{2^p}\right]^{\frac{1}{p}-1}\\
\implies 1 &= 2\infopt_1^{p-1} \left[\infopt_1^p + \frac{1}{2^p}\right]^{\frac{1}{p}-1}\\
\implies \left[\infopt_1^p + \frac{1}{2^p}\right]^{1 - \frac{1}{p}} &= 2\infopt_1^{p-1}\\
\implies \left[\infopt_1^p + \frac{1}{2^p}\right]^{p - 1} &= 2^p \infopt_1^{p(p-1)}\\
\implies \infopt_1^p + \frac{1}{2^p} &= 2^{\frac{p}{p-1}} \infopt_1^{p}\\
\implies \frac{1}{2^p} &= \infopt_1^p\left(2^{\frac{p}{p-1}} - 1\right)\\
\therefore \infopt_1 &= \frac{1}{2} \left[\frac{1}{(2^{\frac{p}{p-1}}-1)}\right]^{\frac{1}{p}}.
\end{align*}
Hence, $\nabla_{\vec f} H^\star(\vec f) = \vec 0$ at $\vec \infopt$. And since $H^\star$ is convex in $[0,1]^2$ by Claim~\ref{claim:limit-convexity}, $\vec \infopt$ is the minimizer in this region.
\hfill$\blacksquare$

\medskip
\noindent\textbf{Proof of Claim~\ref{claim:ssd-inf-order}}\hspace{5px}
For any $p > 1$, we know
$$\frac{p}{p-1} > 1.$$
This implies that
\begin{align*}
    2^{\frac{p}{p-1}} - 1 > 1
    \quad \text{and hence} \quad
    \left[\frac{1}{2^{\frac{p}{p-1}} - 1}\right]^{\frac{1}{p}} < 1.
\end{align*}
Finally, using the values from Claim~\ref{claim:limit-mins}, we obtain
$$\infopt_1 < \infopt_2.$$
\hfill$\blacksquare$

\medskip
\noindent\textbf{Proof of Claim~\ref{claim:z-eps-order}}\hspace{5px}
Let $z \in \mathcal{Z}_\epsilon(1,1)$. Pick an arbitrary $(f_1, f_2) \in [0,1]^2$. As in the proof of Claim~\ref{claim:ssd-limit}, on subtracting Equations~\eqref{eqn:ssd-opt3} and~\eqref{eqn:ssd-limit}, and taking an absolute value, we obtain Equation~\eqref{eqn:H-abs-difference}, that is,
\begin{align}   \label{eqn:difference}
    |H_z(f_1, f_2) - H^\star(f_1, f_2)| = \Big[\big(z-f_1\big)^p + f_2^p\Big]^{\frac{1}{p}} - \big(z - f_1\big).
\end{align}
Combining Equation~\eqref{eqn:difference} with the fact that $0 \leq f_2 \leq 1$, we obtain
\begin{align}   \label{eqn:abs-difference}
|H_z(f_1, f_2) - H^\star(f_1, f_2)| \leq \Big[\big(z-f_1\big)^p + 1\Big]^{\frac{1}{p}} - \big(z - f_1\big).
\end{align}
Now, define $\psi(x) = \left(x^p + 1\right)^{\frac{1}{p}} - x$. We show that $\psi(x)$ is a non-increasing function for $x \geq 0$. Computing the derivative, we have
\begin{align*}
\frac{d \psi(x)}{d x} = x^{p-1}\left(x^p + 1\right)^{\frac{1}{p} - 1} - 1 = \left(\frac{x^p}{x^p + 1}\right)^{\frac{p-1}{p}} - 1 \leq 0 
\end{align*}
for $x \geq 0$, showing that it is a non-increasing function. Going back to Equation~\eqref{eqn:abs-difference}, we know that $f_1 \leq 1$. Therefore, $\big(z-f_1\big) \geq \big(z-1\big) \geq 0$. Using the fact that $\psi$ is a non-increasing function, we obtain $\psi\big(z-f_1\big) \leq \psi\big(z-1\big)$, which on expansion gives us
\begin{align}   \label{eqn:abs-less-than}
\Big[\big(z-f_1\big)^p + 1\Big]^{\frac{1}{p}} - \big(z - f_1\big) \leq \Big[\big(z-1\big)^p + 1\Big]^{\frac{1}{p}} - \big(z - 1\big)
= |H_z(1, 1) - H^\star(1, 1)|.
\end{align}
Combining Equations~\eqref{eqn:abs-difference} and~\eqref{eqn:abs-less-than}, and the fact that $z \in \mathcal{Z}_\epsilon(1,1)$, we obtain
$$|H_z(f_1, f_2) - H^\star(f_1, f_2)| \leq |H_z(1, 1) - H^\star(1, 1)| < \epsilon.$$
Hence, $z \in \mathcal{Z}_\epsilon(f_1, f_2)$.
\hfill$\blacksquare$

\medskip
\noindent\textbf{Proof of Claim~\ref{claim:fn-value-diff}}\hspace{5px}
The proof follows using three facts:
\begin{enumerate}
    \item Equation~\eqref{eqn:epsilon-band} for $\vec \aggr(z)$ says that $H^\star(\vec \aggr(z)) < H_z(\vec \aggr(z)) + \epsilon$.
    \item Because $\vec \aggr(z)$ is the minimizer of $H_z$, we have $H_z(\vec \aggr(z)) \leq H_z(\vec \infopt)$.
    \item For $\vec{\infopt}$, Equation~\eqref{eqn:epsilon-band} gives us $H_z(\vec \infopt) < H^\star(\vec \infopt) + \epsilon$.
\end{enumerate}
Putting these equations together:
$$H^\star(\vec \aggr(z)) < H_z(\vec \aggr(z)) + \epsilon \leq H_z(\vec \infopt) + \epsilon < H^\star(\vec \infopt) + 2\epsilon.$$
\hfill$\blacksquare$

\medskip
\noindent\textbf{Proof of Claim~\ref{claim:level-set-prop}}\hspace{5px}
We prove the claim by contraposition. Pick an arbitrary $\vec f \in [0,1]^2$ such that $\|\vec f - \vec \infopt\|_2 > \frac{\Delta}{4}$. This means that there exists $\vec g \in [0,1]^2$ on the line joining $\vec f$ and $\vec \infopt$ such that $\|\vec g - \vec \infopt\|_2 = \frac{\Delta}{4}$. We could alternatively write $\vec g = \theta \vec f + (1-\theta)\vec \infopt$, where $\theta \in (0,1)$. By convexity of $H^\star$,
\begin{align}   \label{eqn:convexity-middle-point}
    H^\star(\vec g) \leq \theta H^\star(\vec f) + (1-\theta) H^\star(\vec \infopt).
\end{align}
By definition of $\tau$ in~\eqref{eqn:tau-def}, we know $H^\star(\vec g) \geq \tau$. Also, we know $H^\star(\vec \infopt) < \tau$. Using these in~\eqref{eqn:convexity-middle-point}, we obtain
\begin{align*}
    \tau < \theta H^\star(\vec f) + (1-\theta) \tau.
\end{align*}
Therefore, we obtain $H^\star(\vec f) > \tau$. In summary, if $\|
\vec f - \vec \infopt\|_2 > \frac{\Delta}{4}$, then $H^\star(\vec f) > \tau$. Taking the contrapositive gives us the desired result.
\hfill$\blacksquare$

\subsection{Proof of Lemma~\ref{lem:sp-negative}}
\label{app:sp-negative}

Consider \Lpq{p}{q} aggregation with arbitrary $q \in (1, \infty]$. We show that strategyproofness is violated. The construction for this is as follows. Suppose there is one paper $a$ and two reviewers. The first reviewer gives the paper an overall recommendation of $1$ and the second reviewer gives it an overall recommendation of $0$. Let $\crit_a$ be the (objective) criteria scores of this paper. 

Let us first consider $q \in (1,\infty)$. For a function $f:\domX \to \domY$, all we care about in this example is its value at $\crit_a$. Hence, for simplicity, let $f_a$ denote the value of function $f$ at $\crit_a$, i.e, $f_a := f(\crit_a)$.
Then our aggregation becomes
$$\aggr_a = \underset{f_a \in \mathbb{R}}{\text{argmin}}\ \Big\{|1 - f_a|^q + |f_a|^q\Big\}.$$

We claim that $f_a=0.5$ is the unique minimizer. Observe that if $f_a = 0.5$, then the value of our objective is $0.5^q + 0.5^q < 1$ when $q \in (1,\infty)$. On the other hand, if $f_a \geq 1$ or if $f_a \leq 0$ then the value of our objective is at least $1$. Hence $f_a \in (0,1)$. By symmetry, we can restrict attention to the range $[0.5,1)$ since if there is a minimizer in $(0,0.5)$ then there must also be a minimizer in $(0.5,1)$. Consequently, we rewrite the optimization problem as
\begin{align}   \label{eqn:sp-opt}
\aggr_a = \underset{f_a \in [0.5,1)}{\text{argmin}}\ \Big\{(1 - f_a)^q + f_a^q\Big\}.
\end{align}

Consider the function $h:[0.5,1] \to \mathbb{R}$ defined by $h(x) = x^q$. This function is strictly convex (the second derivative is strictly positive in the domain) whenever $q \in (1,\infty)$. Hence from the definition of strict convexity, we have
$$0.5 \big( (1 - f_a)^q + f_a^q \big) > \big( 0.5 (1-f_a + f_a) \big)^q = 0.5^q$$
whenever $f_a \in (0.5,1)$. Consequently, the objective value of~\eqref{eqn:sp-opt} is greater at $f_a \in (0.5,1)$ than at $f_a = 0.5$. We conclude that $\aggr_a = 0.5$ whenever $q \in (1,\infty)$.

When $q = \infty$, we equivalently write the optimization problem as $$\aggr_a = \underset{f_a \in \mathbb{R}}{\text{argmin}}\  \max\big(|1-f_a|, |f_a|\big).$$
This objective has a value of $0.5$ if $f_a = 0.5$ and strictly greater if $f_a \neq 0.5$. Hence, $\aggr_a = 0.5$ for $q = \infty$ as well.

The true overall recommendation of reviewer $1$ differs from the aggregate $\aggr_a$ by $0.5$ (in every $L_\ell$ norm). However, if reviewer $1$ reported an overall recommendation of $2$, then an argument identical to that above shows that the minimizer is $\widehat{g}_a = 1$. Reviewer $1$ has thus successfully brought down the difference between her own true overall recommendation and the aggregate $\widehat{g}_a$ to $0$. We conclude that strategyproofness is violated whenever $q \in (1,\infty]$. \hfill$\blacksquare$

\subsection{Proof of Lemma~\ref{lem:consensus}}
\label{app:consensus}

The construction showing that $L(\infty,1)$ aggregation violates consensus is as follows. Suppose there are two papers, two reviewers and both reviewers review both papers. Assume that the papers have objective criteria scores $\crit_1$ and $\crit_2$, and that neither of these scores is pointwise greater than or equal to the other. Let the overall recommendations of the reviewers for the papers be given by the matrix
\[\mathbf{y} = \begin{bmatrix}
0 & 1\\
2 & 1
\end{bmatrix},\]
where $y_{ia}$ denotes the overall recommendation of reviewer $i$ for paper $a$. Since both reviewers give the same overall recommendation of $1$ to paper $2$, any aggregation method that satisfies consensus must also give paper $2$ an aggregate score of $1$. We show that this is not the case under \Lpq{\infty}{1} aggregation.

Let $f_i$ denote the value of function $f$ on paper $i$, i.e. $f_i := f(\crit_i)$. And let $\aggr_i$ denote the aggregate score of paper $i$. Since we are minimizing \Lpq{\infty}{1} loss, the aggregate function satisfies:
\begin{align}	\label{eqn:consensus-opt}
(\aggr_1, \aggr_2) \in \underset{(f_1, f_2) \in \mathbb{R}^2}{\text{argmin}} \bigg\{\big\|(0,1) - (f_1, f_2)\big\|_\infty + \big\|(2,1) - (f_1, f_2)\big\|_\infty\bigg\}.
\end{align}
We do not have any monotonicity constraints in~\eqref{eqn:consensus-opt} because the two papers have incomparable criteria scores. Denote the objective function of~\eqref{eqn:consensus-opt} by $G(f_1, f_2)$. We can simplify this objective to
\begin{align}   \label{consensus-obj}
G(f_1, f_2) =
\max(|f_1|, |f_2-1|) + \max(|2-f_1|, |f_2-1|).
\end{align}
We claim that $(0.5, 0.5)$ is a minimizer of $G$. The objective function value at this point is
$$G(0.5,0.5) = \max(0.5, 0.5) + \max(1.5, 0.5) = 0.5 + 1.5 = 2.$$
For arbitrary $(f_1, f_2) \in \mathbb{R}^2$, we have
\begin{align*}
G(f_1, f_2) &=
\max(|f_1|, |f_2-1|) + \max(|2-f_1|, |f_2-1|)\\
&\geq |f_1| + |2-f_1|\\
&\geq 2 = G(0.5, 0.5),
\end{align*}
where the first inequality holds because the maximum of two elements is always larger than the first, and the second inequality holds by the triangle inequality. Therefore, $(0.5, 0.5)$ is a minimizer of $G$. The $L_2$ norm of this minimizer is $0.5 \sqrt{2} < 1$. On the other hand, any minimizer $(\aggr_1, \aggr_2)$ with $\aggr_2 = 1$ would have an $L_2$ norm of at least $1$. It follows that such a minimizer will not be selected. In other words, \Lpq{\infty}{1} aggregation would select a minimizer for which the aggregate score of paper $2$ is not $1$, violating consensus.\footnote{Observe that even if we used any $L_k$ norm with $k \in (1, \infty)$ for tie-breaking, the $L_k$ norm of $(0.5, 0.5)$ would be $0.5 \sqrt[k]{2} < 1$, while the $L_k$ norm of any minimizer $(\aggr_1, 1)$ would still be at least $1$, violating consensus.}  \hfill$\blacksquare$

\begin{figure}[h]
\centering
\begin{tikzpicture}[scale=2.0]
\filldraw[fill=blue!50!white, draw=black] (1,0) -- (2,1) -- (1,2) -- (0,1) -- cycle;
\draw[step=0.25cm,gray,very thin] (-0.4,-0.4) grid (2.2,2.2);
\draw[thick,->] (0,0) -- (2.2,0) node[anchor=north west] {};
\draw[thick,->] (0,0) -- (0,2.2) node[anchor=south east] {};
\foreach \x in {0,1,2}
    \draw (\x cm,1pt) -- (\x cm,-1pt) node[anchor=north] {$\x$};
\foreach \y in {0,1,2}
    \draw (1pt,\y cm) -- (-1pt,\y cm) node[anchor=east] {$\y$};
\end{tikzpicture}
\caption{The shaded region depicts the set of all minimizers of~\eqref{eqn:consensus-opt}. $f_1$ is on the x-axis and $f_2$ is on the y-axis.}
\label{fig:consensus-mins}
\end{figure}
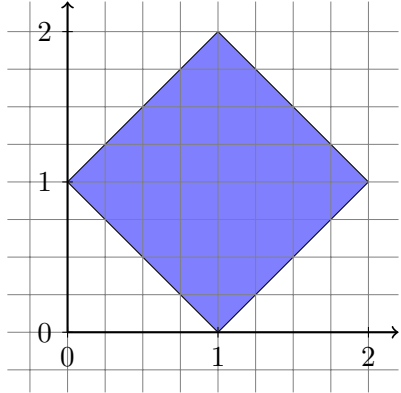

\paragraph{Complete picture of minimizers} For completeness, we look at the set of all minimizers of $G$. This is given by
$$\widehat{F} = \big\{(f_1, f_2) \ | \ f_1 \in [0,2], f_2 \in [1 - \min(f_1, 2 - f_1), 1 + \min(f_1, 2 - f_1)]\big\}.$$
Pictorially, this set is given by the shaded square in Figure~\ref{fig:consensus-mins}. It is the square with vertices at $(0,1)$, $(1,0)$, $(2,1)$ and $(1,2)$.

This shows that almost all minimizers violate consensus. For the specific tie-breaking considered, the minimizer chosen is the one with minimum $L_2$ norm, i.e., the projection of $(0,0)$ onto this square. This gives us $(0.5, 0.5)$, violating consensus.

Observe that tie-breaking using minimum $L_k$ norm, for $k \in (1, \infty]$, also chooses $(0.5, 0.5)$ as the aggregate function, violating consensus. For $k=1$, all points on the line segment $f_1 + f_2 = 1$ ($0 \leq f_1 \leq 1$) would be tied winners, almost all of which violate consensus. Further, even if one uses other reasonable tie-breaking schemes like maximum $L_k$ norm, they suffer from the same issue, i.e., there is a tied winner which violates consensus.

\section{Additional Empirical Results}
\label{app:empirical}

We present some more empirical results in addition to those provided in the main text.

\subsection{Influence of Varying the Hyperparameters}
\label{app:influence}

Although our theoretical results identify \Lpq{1}{1} aggregation as the most desirable, we would like to paint a broader picture by determining how much impact the choice of $p$ and $q$ actually has on selected papers. To this end, we compute the overlap between the papers selected by \Lpq{p}{q} aggregation, for $p,q \in \{1,2,3\}$ (although in general $p$ and $q$ need not be integral, they can be real as well as $\infty$). Table~\ref{tab:overlap-Lpq} shows the overlap between papers selected by \Lpq{p_1}{q_1} and \Lpq{p_2}{q_2}, where the rows represent $(p_1, q_1)$ and columns represent $(p_2, q_2)$. Note that the table is symmetric. 
The results suggest that $q$ has a more significant impact than $p$ on \Lpq{p}{q} aggregation. For instance, \Lpq{1}{1} behaves more similarly to \Lpq{2}{1} and \Lpq{3}{1} than to \Lpq{1}{2} and \Lpq{1}{3}.

\begin{table}[h!]
\centering
\begin{tabular}{|c|c|c|c|c|c|c|c|c|c|}
\hline
    & \textbf{1,1} & \textbf{1,2} & \textbf{1,3} & \textbf{2,1} & \textbf{2,2} & \textbf{2,3} & \textbf{3,1} & \textbf{3,2} & \textbf{3,3} \\ \hline
\textbf{1,1} & 100.0 & 87.5 & 82.7 & 96.1 & 88.0 & 82.6 & 92.3 & 87.5 & 82.1 \\ \hline
\textbf{1,2} & 87.5 & 100.0 & 94.5 & 88.3 & 94.9 & 93.1 & 87.7 & 94.6 & 92.3 \\ \hline
\textbf{1,3} & 82.7 & 94.5 & 100.0 & 84.0 & 92.1 & 95.2 & 83.5 & 91.8 & 94.0 \\ \hline
\textbf{2,1} & 96.1 & 88.3 & 84.0 & 100.0 & 89.8 & 84.4 & 95.7 & 89.5 & 84.0 \\ \hline
\textbf{2,2} & 88.0 & 94.9 & 92.1 & 89.8 & 100.0 & 94.1 & 89.8 & 98.8 & 93.7 \\ \hline
\textbf{2,3} & 82.6 & 93.1 & 95.2 & 84.4 & 94.1 & 100.0 & 84.4 & 94.1 & 98.6 \\ \hline
\textbf{3,1} & 92.3 & 87.7 & 83.5 & 95.7 & 89.8 & 84.4 & 100.0 & 89.7 & 84.0 \\ \hline
\textbf{3,2} & 87.5 & 94.6 & 91.8 & 89.5 & 98.8 & 94.1 & 89.7 & 100.0 & 93.8 \\ \hline
\textbf{3,3} & 82.1 & 92.3 & 94.0 & 84.0 & 93.7 & 98.6 & 84.0 & 93.8 & 100.0 \\ \hline
\end{tabular}\medskip
\caption{Percentage of overlap (in selected papers) between different \Lpq{p}{q} aggregation methods
\label{tab:overlap-Lpq}}
\end{table}

\subsection{Visualizing the Community Aggregate Mapping} 
\label{app:visual}

Our framework is not only useful for computing an aggregate mapping to help in acceptance decisions, but also for understanding the preferences of the community for use in subsequent modeling and research. We illustrate this application by providing some visualizations and interpretations of the aggregate function $\optfn$  obtained from \Lpq{1}{1} aggregation on the IJCAI review data. 

The function $\optfn$ lives in a $5$-dimensional space,  making it hard to visualize the entire aggregate function. Instead, we fix the values of $3$ criteria at a time and plot the function in terms of the remaining two criteria. In all of the visualization and interpretation below, the fixed criteria are set to their respective (marginal) modes: For `quality of writing' the mode is $7$ (715 reviews), for `originality' it is $6$ (826 reviews), for `relevance' it is $8$ (888 reviews), for `significance' it is $5$ (800 reviews), and for `technical quality' it is $6$ (702 reviews). 

The key takeaways from this experiment are as follows. First, writing and relevance do not have a significant influence (Figure~\ref{fig:visual_qualityrelevance}). Really bad writing or relevance is a significant downside, excellent writing or relevance is appreciated, but everything else in between in irrelevant. Second, technical quality and significance exert a high influence (Figure~\ref{fig:visual_significancetechnical}). Moreover, the influence is approximately linear. Third, linear models (i.e., models that are linear in the criteria) are quite popular in machine learning, and our empirical observations reveal that linear models are partially applicable to conference review data\,---\,for some criteria one may indeed assume a linear model, but not for all.

\begin{figure}[h!]
    \centering
    \subfigure[Varying `originality' and `technical quality']{
        \includegraphics[width=0.47\linewidth]{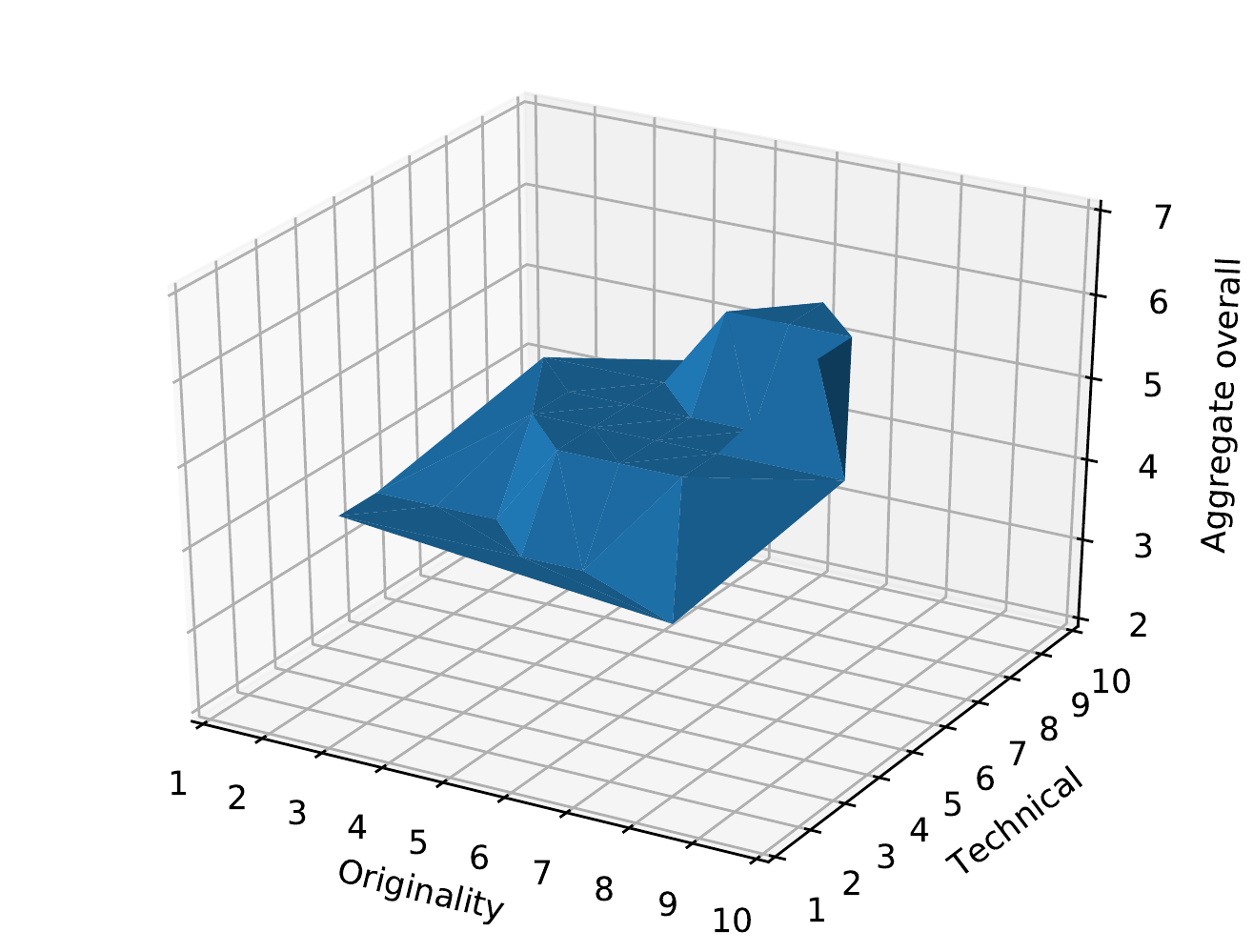} 
        \label{fig:visual6} 
    }
    \subfigure[Varying `originality' and `significance']{
        \includegraphics[width=0.47\linewidth]{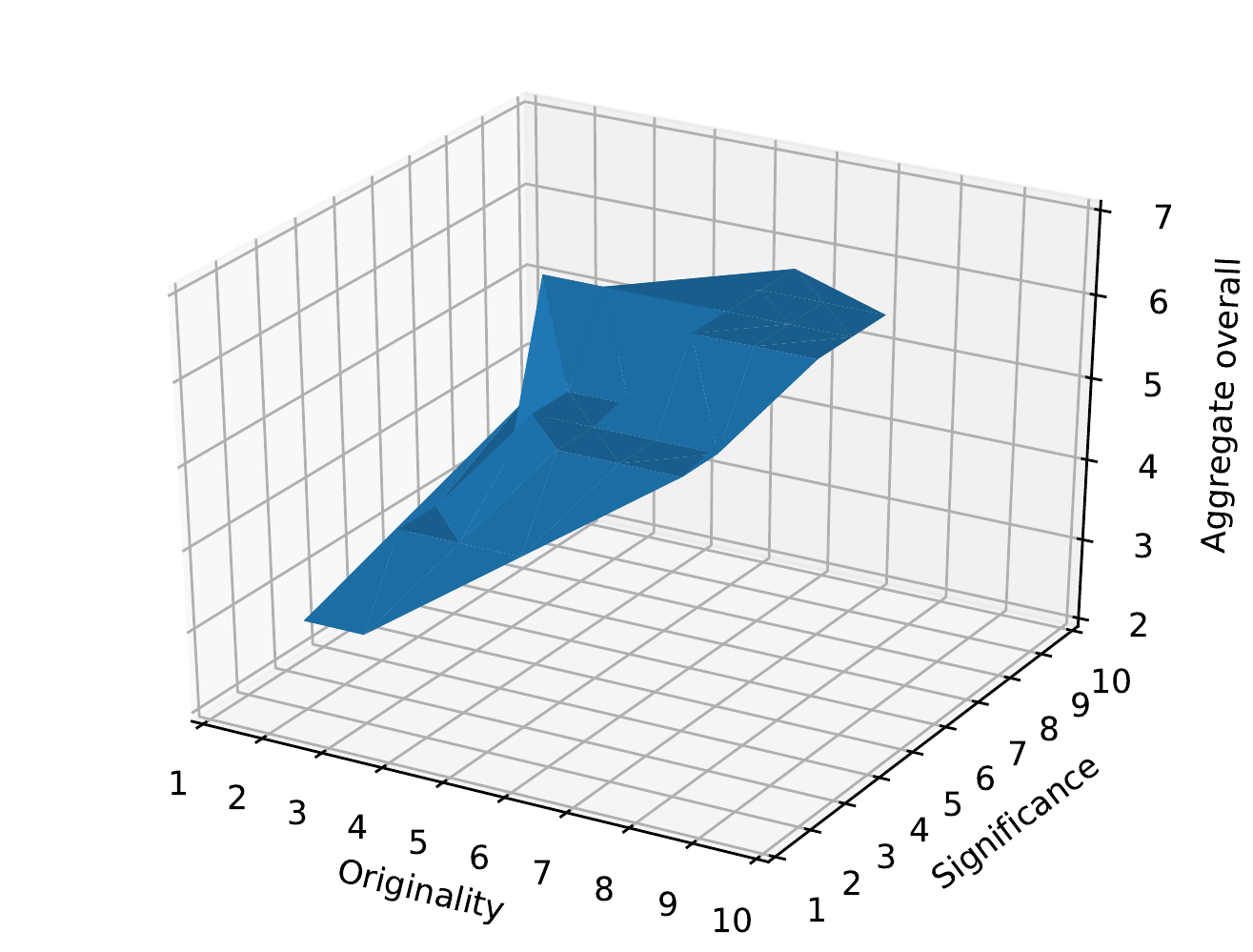} 
        \label{fig:visual7} 
    }
    \subfigure[Varying `quality of writing' and `relevance']
    {
        \includegraphics[width=0.48\columnwidth]{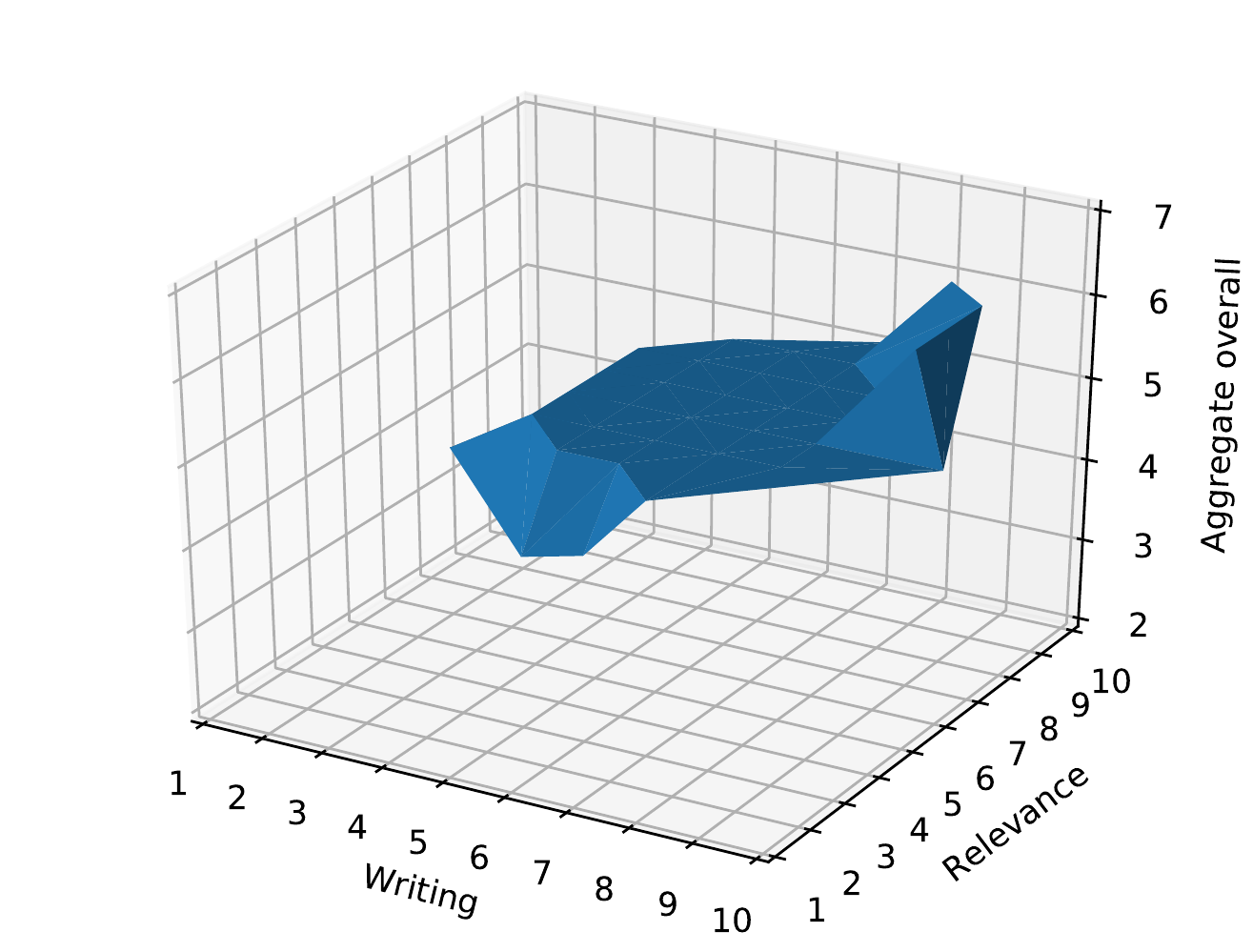} 
        \label{fig:visual_qualityrelevance}
    }
    \subfigure[Varying `significance' and `technical quality']
    {
        \includegraphics[width=0.48\columnwidth]{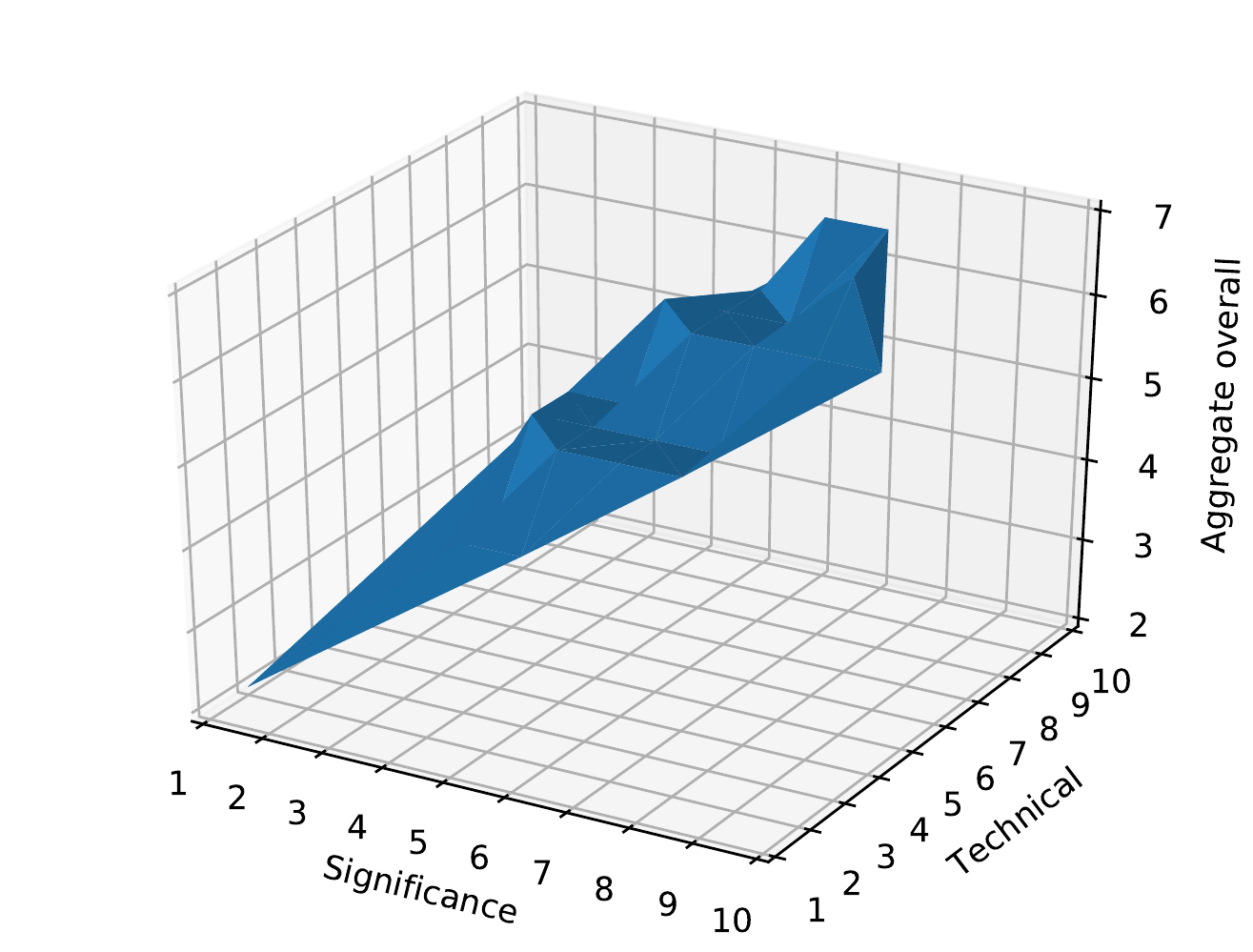} 
        \label{fig:visual_significancetechnical} 
    }
    \caption{Impact of varying different criteria under \Lpq{1}{1} aggregation}
    \label{fig:extras2}
\end{figure}

\begin{figure}[h!]
    \centering
    \subfigure[Varying `relevance' and `technical quality']{
        \includegraphics[width=0.47\linewidth]{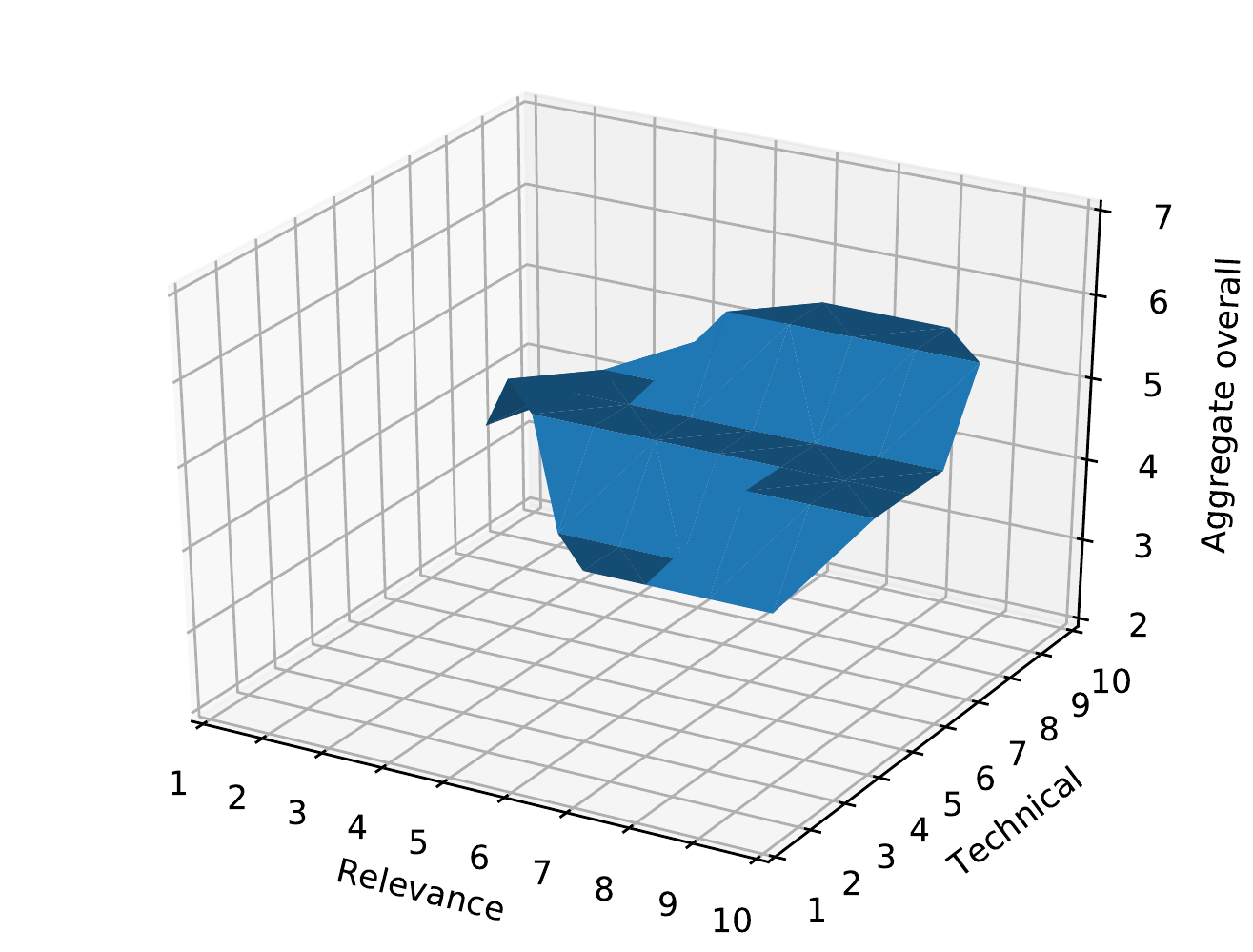} 
        \label{fig:visual4} 
    }
    \subfigure[Varying `relevance' and `significance']{
        \includegraphics[width=0.47\linewidth]{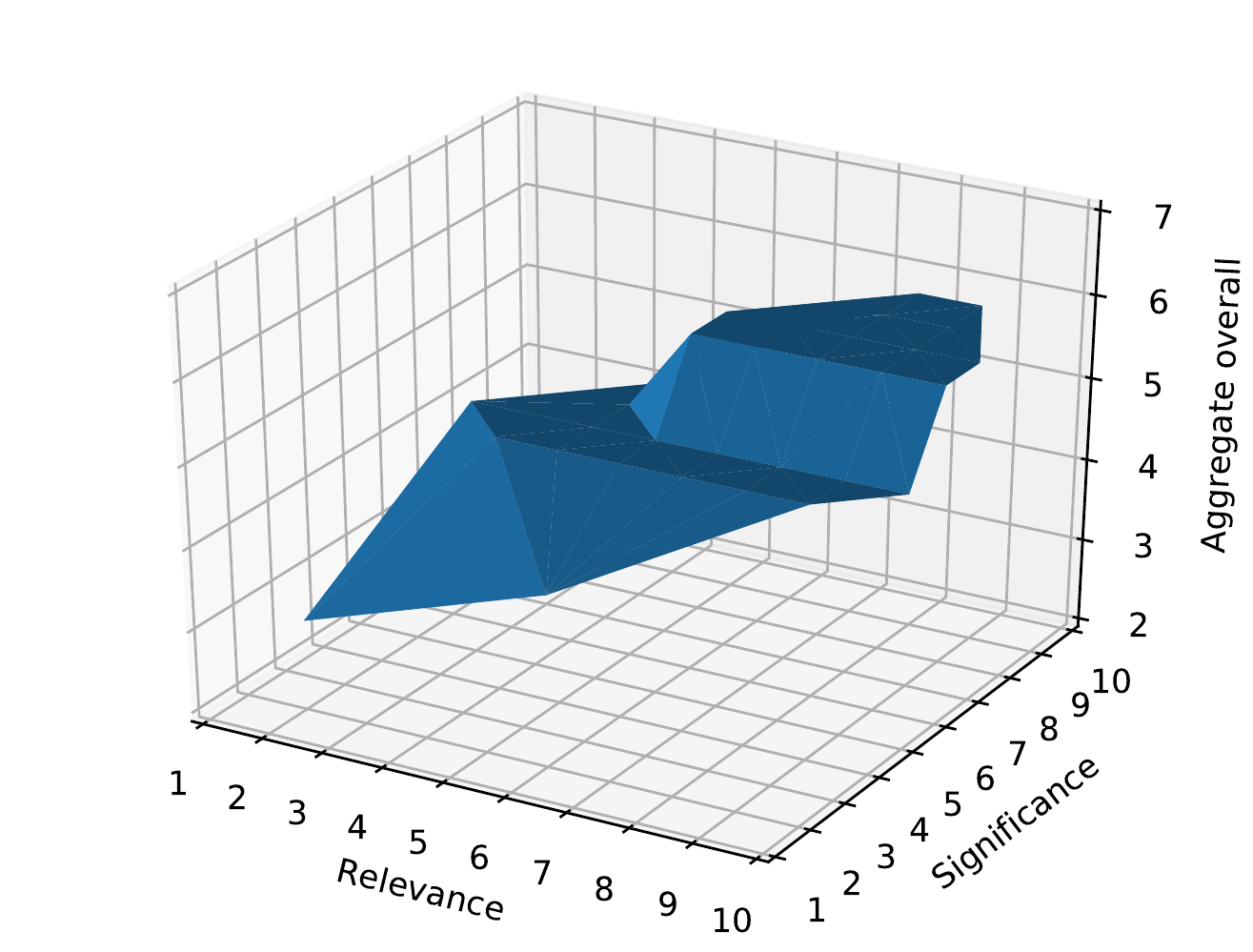} 
        \label{fig:visual5} 
    }
    \subfigure[Varying `quality of writing' and `significance']{
        \includegraphics[width=0.47\linewidth]{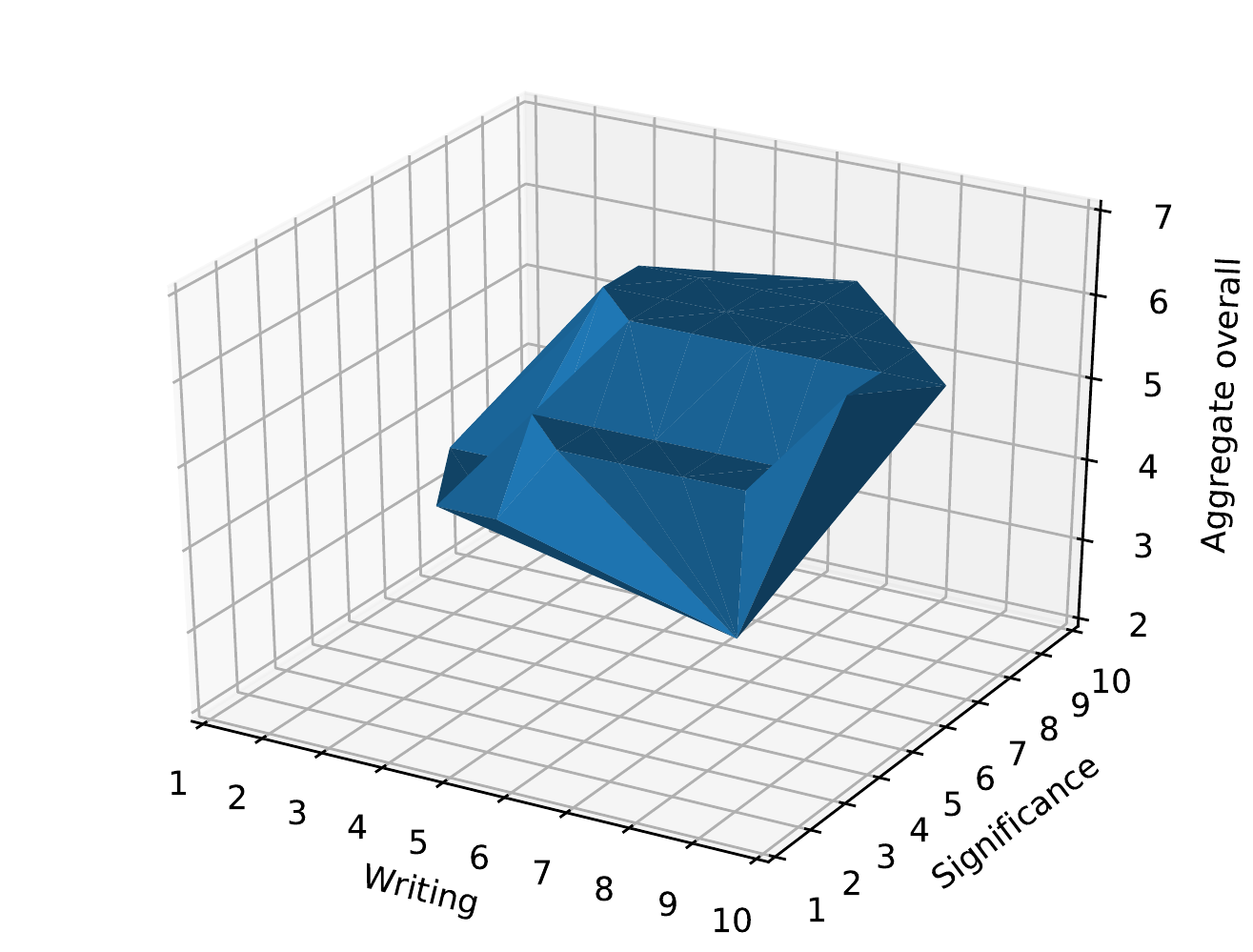} 
        \label{fig:visual8} 
    }
    \subfigure[Varying `quality of writing' and `originality']{
        \includegraphics[width=0.47\linewidth]{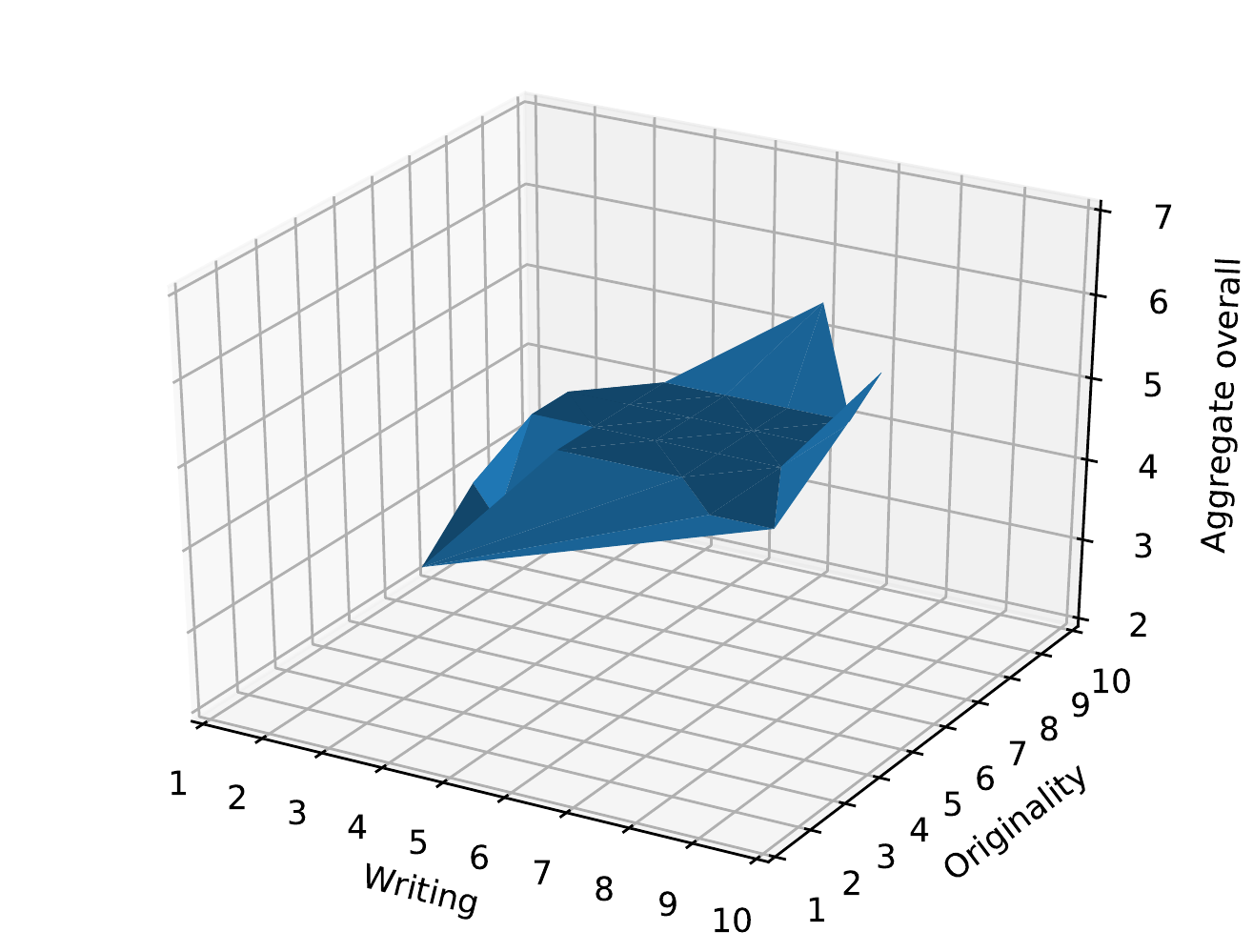} 
        \label{fig:visual10} 
    }
    \subfigure[Varying `quality of writing' and `technical quality']{
        \includegraphics[width=0.47\linewidth]{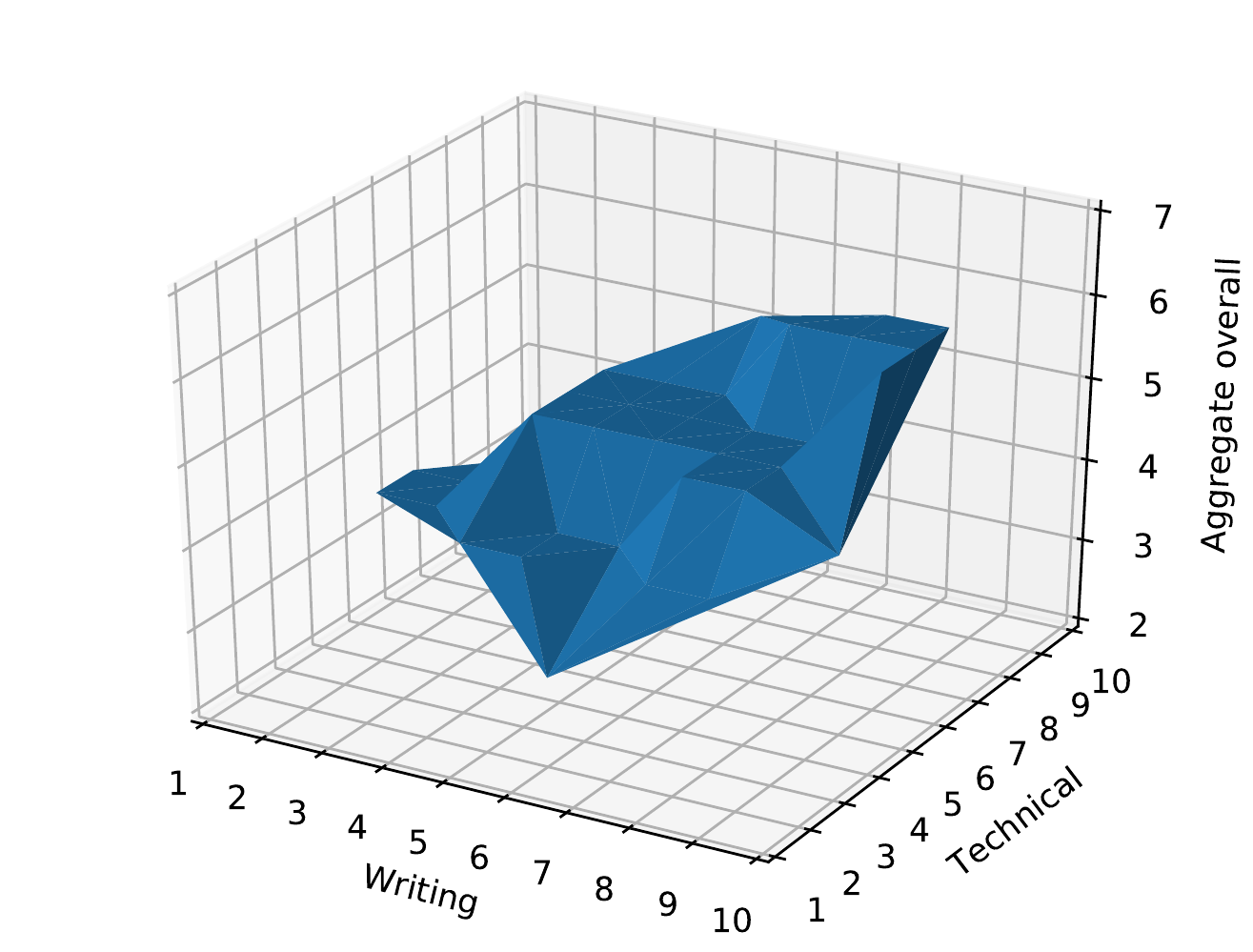} 
        \label{fig:visual1} 
    }
    \subfigure[Varying `originality' and `relevance']{
        \includegraphics[width=0.47\linewidth]{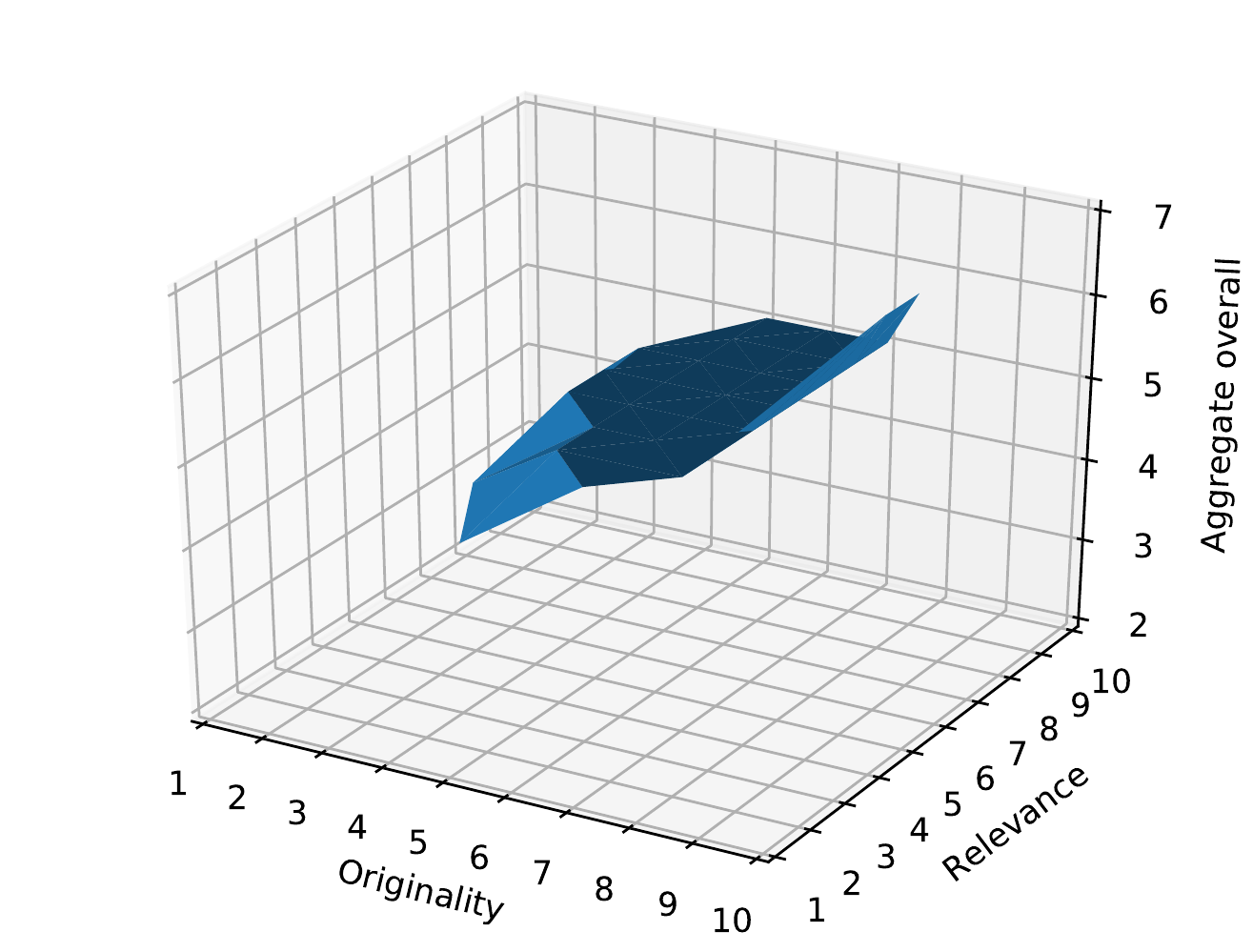} 
        \label{fig:visual2} 
    }
    \caption{Impact of varying different criteria under \Lpq{1}{1} aggregation (continued)}
    \label{fig:extras1}
\end{figure}

\end{document}